%% file: main.tex
\documentclass{article}

\usepackage[margin=1.2in]{geometry}

\usepackage[utf8]{inputenc} 
\usepackage[T1]{fontenc}    
\usepackage{hyperref}       
\usepackage{url}            
\usepackage{booktabs}       
\usepackage{amsfonts}       
\usepackage{nicefrac}       
\usepackage{microtype}      

\usepackage[utf8]{inputenc}

\usepackage{amsmath}
\usepackage{amssymb}
\usepackage{graphicx}
\usepackage{bm}
\usepackage[colorinlistoftodos]{todonotes}
\usepackage{theorem}

\usepackage[title]{appendix}
\usepackage{comment}
\usepackage{amsfonts}
\usepackage{dsfont}
\usepackage{hyperref}
\usepackage{lipsum}
\usepackage{dsfont,subcaption,float}

\usepackage{thm-restate}
\newtheorem{theorem}{Theorem}[section]
\newtheorem{corollary}{Corollary}
\newtheorem{lemma}[theorem]{Lemma}
\newtheorem{remark}{Remark}
\newtheorem{proposition}[theorem]{Proposition}
\newenvironment{proof}{\paragraph{\it Proof.}}{\hfill$\square$}

\newcommand{\E}{\mathbb{E}}	
\newcommand{\D}{\mathcal{D}}	
\newcommand{\indic}{\mathds{1}} 
\newcommand{\vdot}[2]{\langle #1, #2 \rangle}
\newcommand{\Eps}{\mathcal{E}}

\newif\ifdraft
\drafttrue  

\title{EM Converges for a Mixture of Many Linear Regressions}

\author{
 Jeongyeol Kwon \\
   \texttt{kwonchungli@utexas.edu}
   \and
   Constantine Caramanis\\
   \texttt{constantine@utexas.edu} \\
   \and
   The University of Texas at Austin
}

\begin{document}

\maketitle

\allowdisplaybreaks

\begin{abstract}
    We study the convergence of the Expectation-Maximization (EM) algorithm for mixtures of linear regressions with an arbitrary number $k$ of components. We show that as long as signal-to-noise ratio (SNR) is $\tilde{\Omega}(k)$, well-initialized EM converges to the true regression parameters. Previous results for $k \geq 3$ have only established local convergence for the noiseless setting, i.e., where SNR is infinitely large. Our results enlarge the scope to the environment with noises, and notably, we establish a statistical error rate that is independent of the norm (or pairwise distance) of the regression parameters. In particular, our results imply exact recovery as $\sigma \rightarrow 0$, in contrast to most previous local convergence results for EM, where the statistical error scaled with the norm of parameters. Standard moment-method approaches may be applied to guarantee we are in the region where our local convergence guarantees apply.
\end{abstract}

\section{Introduction}
\input{Introduction.tex}

\section{Preliminaries}
\input{Problem_Setup.tex}

\section{Main Results}
\input{Main_Results.tex}

\section{Analysis of Population EM}
\input{Population.tex}

\section{Finite Sample EM Analysis}
\input{Finite_Sample.tex}

\section{Conclusion and Future Works}
In this paper, we provided local convergence guarantees of both population and (sample-splitting) finite-sample EM algorithm for MLR with general $k$ components. For our finite-sample based EM analysis, we decomposed a single random variable into multiple random variables using indicator functions, each of which corresponds to different event. With this strategy, we were able to give a near-optimal statistical error that does not depend on the distances between regression parameters, $R_{min}$ or $R_{max}$. We believe our technique is applicable to other problem settings such as GMM, and other local heuristic algorithms to get an improved statistical error. 

While we studied the local convergence of the EM algorithm in high SNR regime, the question whether EM converges under lower SNR condition, \textit{i.e.} $R_{min} = o(k)$ regime, is still widely open, even when we assume we start from a very good initialization. We do not even know that under this low SNR regime, whether we can recover all parameters with polynomial number of samples (in $k$). Finding a polynomial-time algorithm to find a good initialization is also another challenging problem, if the linear independence assumption between regression vectors does not hold (\textit{e.g.} $d < k$). Studying the EM algorithm in more general settings, \textit{e.g.} with unknown and different covariance for $X$ in each linear model, will be also an interesting future direction.

\bibliographystyle{unsrt}
\bibliography{main}

\begin{appendices}

\input{Appendix.tex}

\end{appendices}

\end{document}

%% file: Introduction.tex
The Expectation-Maximization (EM) algorithm is a powerful tool for statistical inference when we have samples with missing information, often modeled as latent variables. It is a general-purpose heuristic for evaluating the maximum likelihood (ML) estimator for such problems \cite{wu1983convergence}. A canonical example is parameter estimation for the mixture of a known family of parameterized distributions such as Gaussian Mixture Models (GMM) or Mixture of Linear Regressions (MLR). In such problems, solving for maximum likelihood estimator is NP-hard due to the non-convexity of the log-likelihood function. The EM algorithm successively computes tighter lower bounds on the likelihood function; each iteration is no more complex than solving the ML problem with no missing data. Despite its simplicity and broad success in practice, a theoretical understanding of EM remains largely elusive (but see Section \ref{ssec:related} for important recent results). In general, the EM algorithm may fail to converge to a global optimum of log-likelihood function. Thus, its success story is specific to problems to which the EM algorithm is applied.

In this paper, we study the convergence behavior of the EM algorithm for mixture of linear regressions with $k$ components. We show that the EM algorithm converges to the true parameters when the signal-to-noise ratio (SNR) is larger than $\tilde{\Omega}(k)$, and the parameter is well initialized, within $\tilde{O}(1/k)$ of the true parameters (see related works and Remark \ref{remark:tensor_init} for some known initialization techniques). This is the first result, to the best of our knowledge, to establish the convergence of the EM algorithm in MLR with more than two components and finite SNR. Furthermore, under the same regularity conditions, we recover the results of \cite{pmlr-v99-kwon19a} for two-component mixtures, showing that the statistical error of the sample-splitting finite sample EM algorithm is $\tilde{O}(\sigma \sqrt{k^2 d/n})$ where $n$ is the number of samples per iteration. This is significant because our analysis then implies exact recovery in the noiseless setting even with finite number of samples, in contrast to earlier work \cite{balakrishnan2017statistical, klusowski2019estimating} that only showed statistical error scales with the norm (or pairwise distance) of the regression parameters.

\subsection{Related Work}
\label{ssec:related}
\input{Related_Work.tex}

\subsection{Main Contribution}
We prove local convergence of the EM algorithm for $k$-MLR, showing that it converges to a global optimum with high probability, when SNR is $\tilde{\Omega}(k)$, and EM is initialized in a $1/\tilde{O}(k)$-neighborhood of a global optimum. We first establish this result in the infinite sample limit, i.e., population EM. Our result generalizes the results in \cite{balakrishnan2017statistical} which established local convergence for a symmetrized balanced mixture of two components. We establish local convergence in the setting with arbitrary number of components and possibly unbalanced mixing weights. At a high level, our analysis proceeds by carefully constructing the event where the samples are almost correctly assigned their weights, bringing the next estimator closer to the true parameter. Given good initialization and high enough SNR, we expect most samples fall into this category. At the same time, we bound the portion of bad samples which do not fall into this event. The effect of this ``leakage'' is thus canceled out when the average is taken over all samples. By this construction, our convergence rate is no longer dependent on the maximum distance between regression parameters which has often appeared in the EM literature as an artifact of the analysis.

We then show the convergence of a simple variant of finite-sample EM\footnote{The variant is often called {\it sample-splitting}  since it divides entire samples into $T$ batches and uses a new batch in every iteration.} via concentration arguments. Toward this goal, we propose ``event-wise'' concentration of random variables as a proof strategy. Intuitively speaking, the samples that fall into the good event in population EM only induce exponentially small errors. Consequently, statistical errors from these good samples should also be exponentially small. Furthermore, they are the majority among all samples under our assumption on SNR and initialization. On the other hand, samples conditioned on bad events could incur an error as large as the norm of the parameters. However, they are in the minority, and large norms will be canceled out when divided by the total number of samples. See Section \ref{sec:finite_sample} and Proposition \ref{lemma:indic_prob_decompose} for a detailed discussion and formal statement. Remarkably, we show that the statistical error only scales with the variance of the noise.

%% file: Related_Work.tex
Work in \cite{balakrishnan2017statistical} established a characterization of the local region of attraction within which EM is guaranteed to converge to a point with the statistical precision of a global optimum. This complemented work in \cite{yi2014alternating} that gave an analogous result for noise-less mixed regression. A key aspect in \cite{balakrishnan2017statistical} involves coupling an analysis of population EM to finite sample EM. Several results have followed, providing convergence results for canonical problems such as GMM or MLR. In the special case of two balanced mixtures, global convergence results have been established in \cite{jin2016local, daskalakis2017ten} for GMMs, and in \cite{pmlr-v99-kwon19a} for MLR. Beyond more than two components, a negative result for global convergence of the EM algorithm for 3-GMM has been established \cite{jin2016local}, while \cite{zhao2018statistical, yan2017convergence} give a local convergence result for $k$-GMM with arbitrary $k \ge 3$. Attempts have been made to obtain analogous results for mixed linear regression. However, these efforts have only been successful in the setting of infinite SNR, i.e., the noiseless setting. Here, \cite{yi2016solving} establishes convergence of alternating minimization, while \cite{zhong2016mixed} obtain a similar result by solving a non-convex formulation; work in \cite{hand2018convex} gives a convex objective that solves the noiseless MLR problem for well-separated data.

Indeed, the problem of solving mixture of linear regressions has been extensively studied. In general, MLR is NP-hard \cite{yi2014alternating} due to the combinatorial nature of the problem. Therefore, it is natural to consider assumptions in the problem, and various efficient algorithms have been proposed under certain statistical assumptions \cite{sedghi2014provable, chaganty2013spectral, yi2014alternating, chen2014convex, zhong2016mixed, yi2016solving, chen2017convex, li2018learning, hand2018convex}. For instance, \cite{chen2014convex} proposed convex formulation which achieves the optimal minimax rate for equal-weighted 2-MLR, and later in \cite{chen2017convex} extended the treatment to unequally weighted mixtures, but again focus on the mixture of only two components. As mentioned, \cite{yi2016solving, zhong2016mixed, hand2018convex} all propose algorithms for solving $k$-MLR, in the noiseless setting. 

A common technical tool used by many algorithms is the powerful method of moments. In the various algorithms based on method of moments \cite{sedghi2014provable, li2018learning, yi2014alternating, chaganty2013spectral, zhong2016mixed, yi2016solving}, up to third-order tensors are constructed from Gaussian regression models, as all necessary information of the regression vectors are contained in those moments if all regression vectors are linearly independent. The drawback of a purely moment-based method is the high sample and computational complexity. In particular, the statistical error of the resulting estimator typically scales with the norm of the regression parameters. Therefore, these methods are often used in conjunction with fast iterative algorithms, such as gradient descent \cite{zhong2016mixed, li2018learning} or alternating minimization \cite{yi2014alternating, yi2016solving}. While the work cited provides guarantees for these iterative algorithms in the noiseless setting, they are no longer consistent estimators in the presence of noise. In practice, the EM algorithm seems to obtain better results; in theory, however, the question of whether EM always converges to the global optimum for $k$-MLR with $k \ge 3$ is open, even when initialized in a neighborhood of true parameters. This paper provides an affirmative answer to this question.

%% file: Problem_Setup.tex
We consider the mixture of multiple linear regressions, where a pair of random variables $(X, y) \in \mathbb{R}^d \times \mathbb{R}$ are generated from one of $k$ linear models:
$$
    \mathcal{D}_j: y = \vdot{X}{\beta_j^*} + e, \quad  \text{for } \ j = 1, ..., k
$$
where $e$ represents additive noise in the measurement with variance $\sigma^2$. Our goal is recovering regression parameters $\{\beta_j^*\}_{j=1}^k$ when the labels that indicate from which domain each pair is generated are missing. Thus, we are considering the estimation of parameters for the mixture of distributions $\{\mathcal{D}_j\}_{j=1}^k$ with mixing weights $\{\pi_j^*\}_{j=1}^k$. In the finite sample regime, we estimate $\{\beta_j^*\}_{j=1}^k$ when we have $n$ samples $(X_i, y_i)_{i=1}^n \sim \mathcal{D}$, where $\mathcal{D} = \sum_j \pi_j^* \mathcal{D}_j$ is a mixture distribution. 

In this paper, we assume that the design vector $X$ for all linear components comes from a shared standard multivariate Gaussian distribution $\mathcal{N}(0,I_d)$. We assume $e$ is a zero-mean and unit-variance Gaussian random variable and independent of $X$. Thus, the problem is rescaled with known variance parameter $\sigma^2$. 

\subsection{The EM Algorithm}
In general, the EM algorithm performs two steps at each iteration, known as the E-step and M-step; these can be written as follows:
\begin{align*}
    \text{E-step}: &\qquad Q(\bm{\theta} | \bm{\theta}_t) = \mathbb{E}_X [\sum_z p(z|X; \bm{\theta}_t) \log f(X,z; \bm{\theta})], \nonumber \\
    \text{M-step}: &\qquad \bm{\theta}_{t+1} = \arg\max_{\bm{\theta}} Q(\bm{\theta} | \bm{\theta}_t),
\end{align*}
where $f$ is the probability distribution function parametrized by $\bm{\theta}$ that generates the data. Specifically, the E-step forms the likelihood function by assigning conditional probability to hidden labels given samples, based on the current estimator of true parameters. Subsequently, the M-step maximizes the expectation built at the E-step to find a new estimator. The EM algorithm alternates over these two steps iteratively until it converges. Due to the intuitive appeal of updating weights, and then updating the estimates of the $\bm{\theta}_t$, and also due to the computational tractability of each iteration, the EM algorithm has been widely used in many different applications.

When applied to our problem setting, one iteration of the population EM algorithm can be written as:
\begin{align}
    \label{eq:emupdate_population}
    \mbox{E-step}: \qquad w_j &= \frac{\pi_j \exp(-(Y - \vdot{X}{\beta_i})^2/2)}{\sum_{l \in [k]} \pi_l \exp(-(Y - \vdot{X}{\beta_l})^2/2)} \nonumber \\
    \mbox{M-step}: \qquad \beta_j^+ &= (\E_\D [w_j XX^{\top}])^{-1} (\E_\D[w_j XY]), \nonumber \\
    \pi_j^+ &= \E_\D [w_j].
\end{align}
When working with the finite number of samples, each expectation in M-step is replaced with the empirical mean of the corresponding quantity (see equation \eqref{eq:emupdate_finite}). 

\subsection{Notations}
In this paper, $d$ is the dimension of the problem and $k$ the number of components. $(X, Y)$ are a pair of random variables from mixture distribution $\D$, $n$ is the number of samples, and $(X_i, y_i)$ are generated samples. We define pairwise distance $R_{ij}^*$, and $R_{min}, R_{max}$ as the smallest and largest distance between regression vectors of any pair of linear models:
\begin{align*}
    R_{ij}^* = \|\beta_i^* - \beta_j^* \|, R_{min} = \min_{i\neq j} R_{ij}^*, \ R_{max} = \max_{i\neq j} R_{ij}^*.
\end{align*}
We define SNR of this problem as $R_{min}$, which is equivalent to the ratio of minimum pairwise distance versus variance of noise. Define $\rho_\pi = \max_j (\pi_j^*) / \min_j(\pi_j^*)$ as the ratio of maximum mixing weight and minimum mixing weight, and $\pi_{min} = \min_j \pi_j^*$.

We denote the max of two scalar quantities $a,b$ as $(a \vee b)$. When $v$ is a vector, $\|v\|$ is $l_2$ norm of $v$. Inner product of two vectors $u, v$ is denoted as $\vdot{u}{v}$. When $A$ is positive semi-definite (PSD) matrix, $\|A\|_{op} = \sup_{s \in \mathbb{S}^{d-1}} (s^T A s)$ is an operator norm of $A$, where $\mathbb{S}^{d-1}$ represents the unit sphere in $\mathbb{R}^d$ space and $s$ is any unit vector in $\mathbb{R}^d$. We use standard complexity analysis notation $O(\cdot), \tilde{O}(\cdot), \Omega(\cdot)$. We use $\E_P[X]$ to denote the expectation of random variable $X \sim P$. Thus $\E_{\D}[\cdot]$ is the expectation taken over the mixture distribution $\D$, and $\E_{\D_j}[\cdot]$ is the expectation taken over distribution corresponds to $j^{th}$ linear model. We denote $\indic_{X \in \Eps}$ an indicator function for event $\Eps$, and often use a shorthand for it $\indic_{\Eps}$ when the context is clear. We use $\E[X | \Eps]$ to denote conditional expectation under event $X \in \Eps$.

For one step analysis of population EM iteration, we use $\beta_j$ to denote the current estimator of $j^{th}$ parameter, and $\beta_j^+$ to denote the next estimator resulted from EM operator. We denote $\Delta_j := \beta_j - \beta_j^*$. We denote $\tilde{\beta}_j$ and $\tilde{\beta}_j^+$ be corresponding estimators for the finite-sample EM. In the result for entire EM algorithm, $\beta_j^{(t)}$ and $\tilde{\beta}_j^{(t)}$ denote the estimator in the $t^{th}$ step of population EM and finite-sample EM respectively. Notations for mixing weights $\pi_i$ are defined in the similar way. Finally, we use the term with high probability to represent the success probability of the algorithm that is at least $1 - 1/ poly(n)$. 

%% file: Main_Results.tex
\label{section:main_results}

We state the main results for both population EM and finite-sample EM. We provide a proof sketch in the following two sections, and defer details to the Appendix. 

We first state our main convergence result for population EM after $T$ iterations.
\begin{theorem}
    \label{theorem:population_EM}
    There exists universal constant $C, c > 0$ such that if $R_{min} \ge C k \rho_\pi \log^2 (k \rho_\pi)$, $|\pi_j^{(0)} - \pi_j^*| \le \pi_j^* / 2$, and $\| \beta_j^* - \beta_j^{(0)} \| \le c R_{min} / (k \rho_\pi \log (k))$ for all $j$, then EM converges to true parameters after $T = O(\log (\max_j \| \beta_j^* - \beta_j^{(0)}\| / \epsilon))$ steps, {\it i.e.},
    $\max_j \|\beta_j^* - \beta_j^{(T)}\| \le O(\epsilon)$ for all $j$.
\end{theorem}

\begin{remark}[Initialization with Tensor Methods] 
\label{remark:tensor_init}
Tensor-based methods are able to recover all parameters under the condition that the regression parameters are linearly independent\footnote{Without this structural assumption, there is no known polynomial-time algorithm for MLR, hence initialization becomes another challenging open problem.}. However, tensor methods either have a poor dependence on $d$ \cite{chaganty2013spectral}, or a sub-optimal sample-complexity (polynomial) dependence on $R_{max}$, in order to get the precision error independent of $R_{max}$. This is the case for a natural extension of the tensor-based method of \cite{yi2016solving}. Thus it is common procedure to use spectral methods to get a crude but good enough initialization, and then continue with EM when the noise is small. 
\end{remark}

Next, we state our main results for finite-sample EM. In finite-sample EM with sample-splitting strategy, we divide $n$ samples into $T$ batches, and uses a fresh batch of $n/T$ samples per every iteration. To simplify notation, simply use $n$ rather than $n/T$ when it is clear from context that we are focusing on the single stage analysis. For future reference, we write down the update rule for the finite-sample EM.
\begin{align}
    \label{eq:emupdate_finite}
    \mbox{E-step}: \ w_{i,j} &= \frac{\tilde{\pi}_j \exp(-(Y - \vdot{X}{\tilde{\beta}_j})^2/2)}{\sum_{l \in [k]} \tilde{\pi}_l \exp(-(Y - \vdot{X}{\tilde{\beta}_l})^2/2)} \nonumber \\
    \mbox{M-step}: \  \tilde{\beta}_j^+ &= \Big(\frac{1}{n} \sum_{i \in [n]} w_{i,j} X_i X_i^{\top} \Big)^{-1} \Big(\frac{1}{n} \sum_{i \in [n]} w_{i,j} X_i y_i \Big), \nonumber \\ \tilde{\pi}_j^+ &= \frac{1}{n} \sum_{i \in [n]} w_{i,j}.
\end{align}
We show similarly the convergence result for finite-sample EM after $T$ iterations: 
\begin{theorem}
    \label{theorem:finite_EM}
    There exists universal constants $C, c > 0$ such that if $R_{min} \ge C k \rho_\pi \log^2 (k \rho_\pi)$, $|\tilde{\pi}_j^{(0)} - \pi_j^*| \le \pi_j^*/2$, and $\| \beta_j^* - \tilde{\beta}_j^{(0)} \| \le c R_{min} / (k \rho_\pi \log (k))$ for all $j$, then given $n$ i.i.d. samples $(X_i, y_i)$ from mixture distribution $\D$, where the sample complexity is $n/T = \tilde{O}((k/\pi_{min}) (d / \epsilon^2))$, then with high probability, sample-splitting finite-sample EM converges to true parameters, {\it i.e.}, $\|\beta_j^* - \tilde{\beta}_j^{(T)}\| \le O(\epsilon)$ for all $j$ after $T = O(\log (\max_j \| \beta_j^* - \tilde{\beta}_j^{(0)} \|/\epsilon))$ iterations.
\end{theorem}

\begin{remark}
    The statistical error in our result is independent of $R_{min}$ or $R_{max}$. This implies in the original problem where the variance of noise is $\sigma$, we have statistical precision $O(\sigma \epsilon)$. It guarantees exact recovery as $\sigma \to 0$. This is the first result showing that the statistical error rate of the EM algorithm does not depend on the distance between any two regression vectors in noisy environment, as opposed to all previous analysis on EM \cite{balakrishnan2017statistical,klusowski2019estimating,zhao2018statistical, yan2017convergence, yi2015regularized}. We provide the detailed discussion on this issue in Section \ref{sec:finite_sample}.
\end{remark}

\paragraph{Discussion of Main Results.} Several points are in order before we move on to the technical proofs. First, note that in the balanced setting (where all mixing weights are nearly equal to $1/k$), $\rho_\pi = 1$ and $\pi_{min} = 1/k$, thus the SNR condition is $\tilde{\Omega(k)}$, and sample complexity per each iteration is $\tilde{O}(k^2d / \epsilon^2)$. The dependency on $d$ and $\epsilon$ is thus optimal. We note that the $O(k^2)$ dependency appeared even in the noiseless setting \cite{yi2016solving}. The total number of iterations is $T = O(\log (\max_j \|\beta_j^* - \beta_j^{(0)}\| / \epsilon))$, as a result of linear convergence with constant rate in our analysis. In the original scale with noise variance $\sigma^2$, it is equivalent to $T = O(\log (\max_j \|\beta_j^* - \beta_j^{(0)}\|/ \epsilon'))$, to achieve an error of $O(\epsilon')$ where $\epsilon' = \sigma \epsilon$. Note that in the extreme setting where $\sigma \to 0$, exact recovery can still be guaranteed in a finite number of steps, but it requires separate case study on the last iteration which we omit in this paper (see Lemma 3 and Corollary 1 in \cite{yi2016solving}).

A natural question is whether the $\Omega(k)$ requirement for SNR $\Omega(k)$ is sharp. 
In a very closely related problem, the parameter estimation of GMM, \cite{regev2017learning} established the lower bound for minimum separation between the centers of each Gaussian component to recover all centers using a polynomial number of samples. Indeed, the bound $\Omega(\sqrt{\log k})$ established in \cite{regev2017learning} is a threshold above which the labels of most samples can be correctly identified (thus the majority are good samples) if the ground truth parameters are given. However, for mixed linear regression, no such lower bound result has been established. We conjecture that such lower bound might be much larger in mixtures of linear regressions, and it might be closely related to the convergence of the EM algorithm. We leave it as a main future challenge to find such lower bounds for mixed linear regression. In this paper, we focus on the local analysis of the EM algorithm under the condition where the labels of most samples can be correctly identified, if we have good estimate of ground truth parameters. We note that it might be possible to improve the logarithmic factors on the SNR condition with more refined analysis.

Doing away with sample splitting in our algorithm is also a natural and important extension, since we use it in the analysis, though it is well appreciated that in practice EM does not appear to need it.
One way to avoid the sample-splitting technique is to get an uniform concentration bound over local region of interest. Indeed, some previous works on the EM algorithm does precisely this, obtaining uniform concentration of EM operators \cite{yan2017convergence, zhao2018statistical, cai2019chime}. However, their statistical errors have polynomial dependence on $R_{max}$. It is not clear how to remove this in their analysis, even if we do allow sample-splitting. We take an alternate analysis path; while we cannot seem to avoid sample splitting, we do succeed in removing this $R_{max}$ dependence (see Section \ref{sec:finite_sample}). We thus obtain an error rate that is free of distance between any two regression vectors. We leave it as a future work to derive the uniform concentration type result with the same statistical error rate.

%% file: Population.tex
We first give the sketch of the proof for population EM and provide detailed proof in Appendix \ref{Appendix:population_EM}. For the ease of the presentation, we will assume we know the true weights and use them in the main text, while the full proof in Appendix will not assume it. We express $\beta_1^+ - \beta_1^*$ as
\begin{align*}
    \beta_1^+ - \beta_1^* &= (\E_\D [w_1 XX^{\top}])^{-1} (\E_\D[w_1 X(Y - \vdot{X}{\beta_1^*})]).
\end{align*}
Then, we exploit the fact that true parameters are a fixed point of the EM iteration. That is,
\begin{gather*}
    w_1^* = \frac{\pi_1 \exp(-(Y - \vdot{X}{\beta_1^*})^2/2)}{\sum_j \pi_j \exp(-(Y - \vdot{X}{\beta_j^*})^2/2)}, \\
    \E_\D[w_1^* X(Y - \vdot{X}{\beta_1^*})] = \pi_1 \E_{\D_1} [X(Y - \vdot{X}{\beta_1^*})] = 0.
\end{gather*}
Then $\beta_1^+ - \beta_1^*$ can be re-written as
\begin{align*}
    \beta_1^+ - \beta_1^* &= (\E_\D [w_1 XX^{\top}])^{-1} (\E_\D[(w_1 - w_1^*) X(Y - \vdot{X}{\beta_1^*})]) \\
    &= (\E_\D [w_1 XX^{\top}])^{-1} (\E_\D[\Delta_w X(Y - \vdot{X}{\beta_1^*})]),
\end{align*}
where we defined $\Delta_w = w_1 - w_1^*$. We then bound two terms operator norm of $A = \E_\D [w_1 XX^{\top}]$ and $B = \E_\D[\Delta_w X(Y - \vdot{X}{\beta_1^*})]$.

\begin{remark}
    We do not exactly verify the so-called gradient smoothness (GS)-condition for population EM operator as proposed in \cite{balakrishnan2017statistical}. The reason for that becomes clear in the finite-sample analysis: the inverse of $\E_\D[w_1 XX^{\top}]$ does not match that of finite sample EM, which has inverse of $1/n \sum_i w_{1,i} X_i X_i^{\top}$. If we try to control the deviation of the finite-sample EM operator from the population EM operator, this mismatch inevitably results in a statistical error that scales with $R_{max}$. See Section \ref{sec:finite_sample} for more comprehensive discussion.
\end{remark}

\subsection{Bounding $B$}
We will first bound $B$. The high-level idea of bounding $B$ is closely related to \cite{balakrishnan2017statistical}. Our result formalizes the proof idea in \cite{balakrishnan2017statistical} to make it applicable to $k$-mixture of regressions. Define $D_m:= \max_j \|\beta_j - \beta_j^*\|$. We first treat the case when $D_m > 1$, and apply mean-value theorem for $D_m \le 1$. In either case, we construct good events with carefully chosen parameters to bound the portion of bad samples and errors induced by good samples. We start with stating our lemma on the bound of $B$.
\begin{lemma}
    \label{lemma:population_bound_b}
    Under the condition in Theorem \ref{theorem:population_EM}, when $D_m > 1$, there exists universal constants $c_1, c_1' \in (0, 1/8)$ and $c_2, c_3, c_4 > 0$ such that:
    \begin{align*}
        \|B\| &\le \pi_1^* (c_1 + c_2 k \log (k \rho_\pi) / R_{min}) + \sum_{j\neq1} \pi_j^* \left(c_1' / (k \rho_\pi) + c_3 \log (R_{j1}^* k \rho_\pi) / R_{j1}^* + \left(c_4 D_m / R_{j1}^* \right) D_m \right).
    \end{align*}
    When $D_m \le 1$, there exists another universal constants $c_1, c_1' \in (0, 1/8)$ and $c_2, c_3 > 0$ such that:
    \begin{align*}
        \|B\| &\le \pi_1^* (c_1 + c_2 k \log (k \rho_\pi) / R_{min} ) D_m + \sum_{j\neq 1} \pi_j^* \left( c_1' / (k \rho_\pi) + c_3 \log^2 (R_{j1}^* k \rho_\pi) / R_{j1}^* \right) D_m.
    \end{align*} 
\end{lemma}

We will only present how to bound errors from other components $j \neq 1$ in the main text, but the same idea is applied to all other cases. In defining what are the good samples, we need to consider two things: (i) the noise is not abnormally large, (ii) error from class mismatch is large enough to overcome the noise. It can be formalized into the following on three events in $j^{th}$ component for $j \neq 1$:
\begin{align}
    \label{eq:event_good}
    & \Eps_{j,1} = \{|e| \le \tau_j\}, \ \Eps_{j,2} = \{|\vdot{X}{\Delta_1}| \vee |\vdot{X}{\Delta_j}| \le |\vdot{X}{(\beta_j^* - \beta_1^*)}|/4 \}, \nonumber \\
    & \Eps_{j,3} = \{|\vdot{X}{\beta_j^* - \beta_1^*}| \ge 4\sqrt{2} \tau_j \}, \Eps_{j,good} = \Eps_{j,1} \cap \Eps_{j,2} \cap \Eps_{j,3}.
\end{align}
Here, $\tau_j$ is a threshold parameter that we specify carefully in the proof. When these three events occur at the same time, it is a good sample: weights given to first component for this sample is almost 0. In fact, we can show that $|\Delta_w| \le (\pi_1^* / \pi_j^*) \exp(-\tau_j^2)$. The errors from this event can be thus bounded by 
$$\|\E_{\D_j} [\Delta_w X(Y - \langle X, \beta_1^* \rangle) \indic_{good}]\| \le (\pi_1^* / \pi_j^*) \exp(-\tau_j^2) \sup_{s \in \mathbb{S}^{d-1}} \E_{\D_j} [|\langle X, s \rangle (Y - \langle X, \beta_1^* \rangle) |].$$
As the supremum is shown to be in order $O(R_{j1}^*)$, the choice of $\tau_j = \Theta \left( \sqrt{\log (R_{j1}^* k)} \right)$ here comes clear if we want the error less than $O\left((\pi_1^* / \pi_j^*) / k \right)$.

When one of the above events are violated, we have no control on the weights from wrong components. However, we can instead control the portion of these bad samples. For instance, consider $\Eps_1$ is violated, {\it i.e.}, measurement noise happens to be large. A Gaussian tail bound gives $P(\Eps_1^c) \le 2\exp(-\tau_j^2/2)$. This small probability is then used to bound the error from bad events,
\begin{align*}
    \|\E_{\D_j}[\Delta_w X(Y - \langle X, \beta_1^* \rangle) \indic_{\Eps_1^c}]\| &\le P(\Eps_1^c) \sup_{s \in \mathbb{S}^{d-1}} \E_{D_j}[|\vdot{X}{s} (Y - \vdot{X}{\beta_1^*})| | \Eps_1^c].
\end{align*}
We are left with bounding the expectation conditioned on $\Eps_1^c$, which turns out to be $O(R_{j1}^*)$. With the choice of $\tau_j = \Theta \left(\sqrt{\log(R_{j1}^* k \rho_\pi)} \right)$, this term is bounded by small value. The rest of the proof follows similarly. The complete proof for $D_m \ge 1$ can be found in Appendix \ref{Appendix:population_EM_B}. 

When $D_m \le 1$, we first apply the mean-value theorem to get a tighter error bound that is proportional to $D_m$. Then we can construct similar events that define good samples, and apply the same approach. Since $D_m \le 1$ case involves heavy algebraic manipulation, we defer the proof for this case until Appendix \ref{appendix:population_EM_B_II}.

\begin{remark}[Unknown mixing weights] 
\label{remark:unknown_mixing_weights}
Our analysis also shows that EM succeeds when it is simultaneously estimating mixing weights and parameters. Tensor-based methods can also provide a good estimation of these mixing weights along with the estimate of regression vectors when they are not known in advance. In order to handle unknown mixing weights, the only additional requirement is the initial guess of mixing weights to be close in relative scale, {\it i.e.}, $|\pi_j - \pi_j^*| \le c \pi_j^*$ for some small $c > 0$, which we set $1/2$ as a requirement for the initial estimator.

When $D_m \ge 1$, the only change in the proof is replacing $\pi_j^*$ with $\pi_j$, {\it i.e.,} using the estimator of weights instead of true weights. In this regime, it is enough to show that $\pi_j^+$ stays in the neighborhood of true mixing weights, i.e., $|\pi_j^+ - \pi_j^*| \le \pi_j^*/2$. When $D_m \le 1$, when we apply mean-value theorem, there is an additional term differentiated by mixing weights, which can also be bounded using the same approach. In this regime, there is also an improvement over mixing weights, i.e., $\max_j |\pi_j^+ - \pi_j^*|/\pi_j^* \le \gamma \max(\max_j |\pi_j - \pi_j^*|/\pi_j^*, D_m)$ for some $\gamma < 1/2$ under the SNR and initialization condition we assume. Proofs for unknown mixing weights are given in Appendix \ref{Appendix:population_EM_B} for case $D_m \ge 1$, and Appendix \ref{appendix:population_EM_B_II} for case $D_m \le 1$.
\end{remark}

\subsection{Bounding $A$}
We will give a lower bound on the minimum eigenvalue of $\E_\D[w_1 X X^{\top}]$. We first observe that 
\begin{align*}
    \E_\D[w_1 X X^{\top}] = \sum_j \pi_j^* \E_{\D_j} [w_j XX^{\top}] \succeq \pi_1^* \E_{\D_1} [w_1 XX^{\top}].
\end{align*}
Then the right hand side is lower bounded as follows:
\begin{lemma}
    \label{lemma:population_bound_a}
    There exists universal constants $c_1 \in (0, 1/2)$ and $c_2, c_3 > 0$, such that:
    \begin{equation*}
        \label{eq:bound_for_A}
        \lambda_{min}(\E_{\D_1}[w_1 XX^T]) \ge 1 - \left(c_1  + c_2 (k \log k) D_m / R_{min} + c_3 (k \log^{3/2} (k\rho_\pi)) / R_{min} \right).
    \end{equation*}
\end{lemma}
Thus, given good initialization $D_m/R_{min} = 1/\tilde{O}(k)$ and SNR $R_{min} = \tilde{\Omega} (k)$, we have $A \succeq (\pi_1^*/2) I$. The detailed proof including the lemma can be found in Appendix \ref{Appendix:population_EM_A}.

\paragraph{\it Proof of Theorem \ref{theorem:population_EM}.} From Lemma \ref{lemma:population_bound_b}, given $R_{min} = C k\rho_{\pi} \log^2 (k \rho_{\pi}) = \tilde{\Omega}(k \rho_{\pi})$ and $D_m = c R_{min}/(k\rho_{\pi} \log k)$, we have $\|B\| \le (\pi_1^* / 4) D_m$ with proper universal constant $C, c > 0$. Similarly, from Lemma \ref{lemma:population_bound_a}. we get $\|A^{-1}\|_{op} \le 2 / \pi_1^*$. Then
$$
    D_m^+ := \max_j \|\beta_j^+ - \beta_j^*\| \le \|A^{-1}\|_{op} \|B\|_2 \le D_m / 2.
$$
We can conclude that after $T = O(\log (\max_j \|\beta_j^{(0)} - \beta_j^*\| / \epsilon)$ iterations, we have $\max_j \|\beta_j^{(t)} - \beta_j^*\| = O(\epsilon)$, thus we get Theorem \ref{theorem:population_EM}. \hfill $\square$

%% file: Finite_Sample.tex
\label{sec:finite_sample}

In the finite sample version of EM, we can derive the estimation error at the next iteration from \eqref{eq:emupdate_finite}:
\begin{align*}
    \beta_1^+ - \beta_1^* = (\sum_i w_{1,i} X_i X_i^{\top})^{-1} (\sum_i w_{1,i} X_i(y_i - \vdot{X_i}{\beta_1^*})).
\end{align*}
To couple it with population EM, we rearrange and write as
\begin{align*}
    \beta_1^+ - \beta_1^* = \Big(\underbrace{ \frac{1}{n} \sum_i w_{1,i} X_i X_i^{\top}}_{A_n} \Big)^{-1} &\Bigg( \underbrace{\Big(\frac{1}{n} \sum_i w_{1,i} X_i(y_i - \vdot{X_i}{\beta_1^*}) - \E_\D[w_1 X(Y - \vdot{X}{\beta_1^*})] \Big)}_{e_B} \\
    &+ \underbrace{\left(\E_\D[w_1 X(Y - \vdot{X}{\beta_1^*})] - \E_\D[w_1^* X(Y - \vdot{X}{\beta_1^*})]\right)}_{B} \Bigg).
\end{align*}
In the analysis of population EM, we show that $B \le c_B D_m \pi_1^*$ for some universal constant $c_B$. Thus, we only have to bound $e_B$, which is the deviation of finite sample mean from true mean $B$. Then we analyze the norm of $A_n$ similarly by relating it to $A$. We focus on the concentration of sums in one-step iteration of EM. We assume that we use sample-splitting finite sample EM as we defined in Section \ref{section:main_results}, and we run EM for $T$ iterations. 

Before getting into our finite-sample analysis, we discuss briefly why we do not use a simpler standard concentration argument. Note that our target for giving a concentration bound is the random variable $w_1 X (Y - \vdot{X}{\beta_1^*})$. On its own, it is a sub-exponential random variable, since $|w_1| \le 1$, $X$ is sub-Gaussian (vector) with parameter $O(1)$, and $Y - \vdot{X}{\beta_1^*}$ is also sub-Gaussian with parameter at most $1 + R_{max}$. Thus, we can apply well-known sub-exponential tail bounds with parameter $O(R_{max})$, and a standard $1/2$ covering-net argument over the unit sphere to get a high probability guarantee. However, in this manner, we can only get a $O(R_{max} \sqrt{d/n})$ deviation of sample mean from true mean. 

Most previous results established on finite-sample EM analysis have this dependency on $R_{max}$ for statistical error \cite{balakrishnan2017statistical, yi2015regularized, klusowski2019estimating, yan2017convergence, zhao2018statistical}. In truth, however, this is an artifact of analysis and not a real phenomenon: the true statistical precision is $O(\sqrt{d/n})$ when noise is comparably less than $R_{max}$. For instance, in the extreme scenario, \cite{yi2016solving} established exact recovery guarantee of EM in a noiseless setting, though it has not been obvious how to generalize their analysis to involve some level of noise. 

Now we turn our attention to give a bound for $e_B$, which is given by the following lemma:
\begin{lemma}
    \label{lemma:finite_bound_b}
    Suppose SNR condition $R_{min} \ge C k \rho_\pi \log (k \rho_\pi)$ with sufficiently large $C>0$ and initialization condition $D_m \le c R_{min} / (k \rho_\pi \log(k \rho_\pi))$ for sufficiently small $c > 0$. With sample complexity 
    $$n \ge C' \left((k/\pi_{min}) (d/\epsilon^2) \log^2 (d k^2T/\delta) \right) \vee (k^2T/\delta)^{1/3} = \tilde{\Omega}((k/\pi_{min}) (d/\epsilon^2)),$$
    for some large absolute constant $C' > 0$, we get $\|e_B\| \le D_m \epsilon \pi_1^* + \epsilon \pi_1^*$, with probability at least $1 - \delta/kT$.
\end{lemma}

Our proof strategy to get a sharp concentration result is to partition random variables using indicator functions for disjoint events. Let $\Eps_{j}$ be the event that the the sample comes from the $j^{th}$ component and $j \neq 1$. Then, consider events as in \eqref{eq:event_good} in the population EM. We then decompose each sample using the indicator functions of these events. For simplicity of notation, let $W_{i} = w_{1,i} X_i (Y - \vdot{X_i}{\beta_1^*})$. We can decompose $W_i$ as, for instance,
\begin{align*}
    w_{1,i}X_i(y_i - \vdot{X_i}{\beta_1^*}) = &\sum_{j=1}^k \Big(W_{i} \indic_{\Eps_j \cap \Eps_{j,good}} + W_{i} \indic_{\Eps_j \cap \Eps_{j,1}^c} +
    W_{i} \indic_{\Eps_j \cap \Eps_{j,1} \cap \Eps_{j,2}^c} + W_{i} \indic_{\Eps_j \cap \Eps_{j,1} \cap \Eps_{j,2} \cap \Eps_{j,3}^c} \Big),
\end{align*}
using the definition of events defined in \eqref{eq:event_good}. Then we can provide a finite-sample analysis with the following proposition.
\begin{proposition}
    \label{proposition:indic_prob_union}
    Let $X$ be some random variable and consider a set of disjoint events $A_1, ..., A_m$, such that $P(\cup_{i=1}^m A_i) = 1$. Then,
    \begin{align*}
        P(|X - \E[X]| \ge t) \le \sum_{i=1}^m P(|X \indic_{A_i} - \E[X \indic_{A_i}]| \ge t_i),
    \end{align*}
    for $\sum_{i=1}^m t_i = t$. 
\end{proposition}
This is a simple restatement of the elementary union bound. It tells us that we can bound tail probabilities of decomposed random variables separately, and then collect them. If for all $i$, $P(|X\indic_{A_i} - \E[X\indic_{A_i}]| \ge t_i) \le \delta / m$, then $P(|X - \E[X] \ge t) \le \delta$. Note that this decomposition is only for the analysis purpose and does not affect the practical implementation of the EM algorithm.

The next proposition is the key ingredient for giving a sharp concentration on each decomposed random variable.
\begin{restatable}{lem}{probdecomp}
\begin{proposition}
    \label{lemma:indic_prob_decompose}
    Let $X$ be some random $d$-dimensional vector, and $A$ be the event with $p = P(A) > 0$. Let random variable $Y = X|A$, {\em i.e.}, $X$ conditioned on event $A$, and $Z = \indic_{X \in A}$. Let $X_i, Y_i, Z_i$ be the i.i.d. samples from corresponding distributions. Then,
    
    \begin{align}
        \label{eq:conditional_decomposed_prob}
        P\Bigg( \Big\|\frac{1}{n} \sum_{i=1}^{n} X_i \indic_{X_i \in A} - &\E[X \indic_{X \in A}] \Big\| \ge t \Bigg) \le \max_{m \le n_e} P \left(\frac{1}{n} \left\|\sum_{i=1}^{m} (Y_i - \E[Y]) \right\| \ge t_1 \right) \nonumber \\
        & + P\left(\|\E[Y]\| \left|\frac{1}{n} \sum_{i=1}^n Z_i - p\right| \ge t_2 \right)
        + P\left(\left|\sum_{i=1}^n Z_i\right| \ge n_e+1\right),
    \end{align}
    for any $0 \le n_e < n$ and $t_1 + t_2 = t$.
\end{proposition}
\end{restatable}
\begin{remark}
    The decomposition of \eqref{eq:conditional_decomposed_prob} is key for obtaining right error bounds for the statistical error. The overall idea is natural. Intuitively, statistical error in empirical mean of $X_i \indic_{X_i \in A}$ consists of two terms: deviation of conditional sums of $X_i | A$, and deviation of the number of bad samples, from their respective expectations. As $X_i\indic_{X_i \in A}$ will be 0 for most samples, we expect the empirical mean will be also very small (and thus, small statistical error). The proposed decomposition obtains a tight error bound by taking the rare probability into account by counting the bad samples first. Proofs for Proposition \ref{lemma:indic_prob_decompose} is in Appendix \ref{appendix:technical_lemmas}. 
\end{remark}

Proposition \ref{lemma:indic_prob_decompose} helps us to accurately control the concentration of random vectors under different events: in a major event where samples are good, we know that $w_{i,1}$ is almost always exponentially small $O(\exp(-\tau_j^2))$ when the sample did not come from the first model. Therefore, norm of $W_i$ conditioned on good event can be controlled with tiny $w_{i,1}$ (see Appendix \ref{appendix:concentrate_B} for precise construction). 

On bad events, such as when measure noise happens to be very large, the (sub-exponential) norm of $W_{i}$ may be as large as $R_{max}$, since the weights of the wrong components could be away from zero. Fortunately, we can survive from these errors due to low chance of bad events given large SNR and good initialization. It enables us to choose $n_e$ small enough while suppressing $P(|\sum_i^n Z_i| \ge n_e + 1)$, and $n_e / n$ cancels out large norm of $W_{i}$ conditioned on bad but rare events. This technique is critical not only for removing the dependency on $R_{max}$, but also obtaining as small dependency on $k$ and $\pi_{min}$ as possible.

We see in the proof that $W_{i}$ conditioned on each event is another sub-exponential random vectors with different sub-exponential norm. Therefore, we can give a sharp concentration bound on every decomposed random variable separately. The full proof of concentration results including Lemma \ref{lemma:finite_bound_b} is in Appendix \ref{appendix:concentration_finite_sample}.

\paragraph{\it Proof of Theorem \ref{theorem:finite_EM}.} Given $\|e_B\| \le D_m \epsilon \pi_1 + \epsilon \pi_1$ and results from population EM, we are left with bounding $A_n := (1/n \sum_{i\in[n]} w_{i,j} X_i X_i^{\top})$. This task can be achieved via a direct application of standard concentration arguments for random matrices \cite{vershynin2010introduction}. When $(1/n \sum_{i\in[n]} w_{i,1} X_i X_i^{\top})$ concentrates well around $\E_{\D}[w_1 XX^{\top}]$ in operator norm, we can conclude that the lower bound of minimum eigenvalue of sample covariance is also lower-bounded by $\pi_1^*/2$, which implies $\|A_n^{-1}\|_{op} \le 2/\pi_1^*$ (see Appendix \ref{appendix:concentrate_A} for details).

Then combining two results, we can conclude that
$$
    \|A_n^{-1}\|_{op} \|B_n\| \le 2{\pi_1^*}^{-1} (c_B D_m \pi_1^* + c_1 D_m \epsilon \pi_1^* + c_2 \epsilon \pi_1^*) \le \gamma_n D_m  + O(\epsilon),
$$
for some $\gamma_n < 1/2$ and universal constant $c_B, c_1, c_2$. This result holds for $1^{st}$ component with probability $1 - \delta/kT$. We can get a same result for other components, and thus we can take a union bound over $k$ components. Thus, we have shown that with probability at least $1 - \delta/T$
$$
    \max_j \| \beta_j^+ - \beta_j^* \| \le \gamma_n \max_j \| \beta_j - \beta_j^* \|+ O(\epsilon).
$$
Iterating over $T$ iterations yields Theorem \ref{theorem:finite_EM}. \hfill $\square$ 

\begin{remark}[Extension to unknown mixing weights with finite-sample EM]
    Proof for mixing weights are also based on the same idea using Proposition \ref{lemma:indic_prob_decompose}. Mixing weights will also be well concentrated in relative scale, {\it i.e.,} $|\tilde{\pi}_1 - \pi_1| \le O(\epsilon) \pi_1^*$. In the finite sample regime, mixing weights might not be exactly recovered even if noise power goes to 0. This does not conflict with the exact recovery guarantee for $\beta$ whose statistical error is proportional to $\sigma$ that goes to 0, since when $\max_j \|\tilde{\beta}_j - \beta_j^*\| \ge \sigma$ or $D_m \ge 1$, we do not require mixing weights to be very close (we only require $|\pi_1 - \pi_1^*| \le \pi_1^*/2$) to get an improved estimator after one EM iteration. In other words, we do not require exact value of mixing weights $\pi$ in order to get exact regression parameters $\beta$. Note that in the noiseless setting, we are always in $D_m \ge 1$ regime. Proof for concentration of mixing weights are given in Appendix \ref{appendix:concentrate_weights}.
\end{remark}

%% file: Appendix.tex
\section{Proofs for Population EM}
\label{Appendix:population_EM}
Throughout the proof, we will use $C, c, c', c_{any}$ without explicit mention whenever we need universal constants to bound any terms. 

Before getting into detailed proofs, we state two essential lemmas from \cite{balakrishnan2017statistical, yi2016solving}.

\begin{restatable}{lem}{anglecomp}
\begin{lemma} [Lemma 6, 7 in \cite{yi2016solving}]
    \label{lemma:angle_comparison}
    Let $X \sim \mathcal{N}(0,I_d)$. For any fixed vector $v \in \mathbb{R}^d$, and a set of vectors $u_1, ... ,u_{k-1} \in \mathbb{R}^d$ such that $\|u_j\| \ge \|v\|$ for all $j$, we define 
    $$
        \Eps := \{|\vdot{X}{u_j}| \ge |\vdot{X}{v}|, \ \forall j=1,...,k-1 \}.
    $$
    Then, 
    \begin{equation}
        \label{eq:aux_angle_1}
        P(\Eps^c) \le \sum_{j=1}^{k-1} \frac{\|v\|}{\|u_j\|}.
    \end{equation}
    
    Furthermore, for any unit vector $s \in \mathbb{S}^{d-1}$ and for any $p \ge 1$, we have
    \begin{align}
        \label{eq:aux_angle_2}
        \E[|\vdot{X}{s}|^p | \Eps^c] &\le k 2^p \Gamma(1 + p/2),
    \end{align}
    where $\Gamma$ is a gamma function.
\end{lemma}
\end{restatable}

\begin{restatable}{lem}{thresholdcomp}
    \begin{lemma} [Lemma 9(v) and 10 in \cite{balakrishnan2017statistical}]
        \label{lemma:threshold comparison}
        Let $X \sim \mathcal{N}(0,I_d)$. For any set of fixed vectors $u_1, ..., u_k \in \mathbb{R}^d$, and fixed constants $\alpha_1, ..., \alpha_k > 0$, define
        $$
            \Eps := \{|\vdot{X}{u_j}| \ge \alpha_j, \ \forall j=1,...,k \}.
        $$
        Then, 
        \begin{equation}
            \label{eq:aux_threshold_1}
            P(\Eps^c) \le \sum_{j=1}^k \frac{\alpha_j}{\|u_j\|}.
        \end{equation}
        
        Furthermore, for any unit vector $s \in \mathbb{S}^{d-1}$ and for $p \ge 1$, we have
        \begin{align}
            \label{eq:aux_threshold_2}
            \E[|\vdot{X}{s}|^p | \Eps^c] &\le k 2^p \Gamma((1 + p)/2) / \sqrt{\pi}.
        \end{align}
    \end{lemma}
\end{restatable}

Proofs of these lemmas can be found in Appendix \ref{appendix:technical_lemmas}. As a consequence of Lemma \ref{lemma:angle_comparison}, \ref{lemma:threshold comparison}, we can show the following corollary.
\begin{corollary}
    \label{corollary:key_bound}
    Under the setting in either Lemma \ref{lemma:angle_comparison} or \ref{lemma:threshold comparison}, we have
    \begin{align*}
        \sqrt{\E[|\vdot{X}{s}|^2 | \Eps^c]} \le C \sqrt{\log k},
    \end{align*}
    for some universal constant $C > 0$.
\end{corollary}
\begin{proof}
    We show for Lemma \ref{lemma:angle_comparison} first. By Holder's inequality,
    \begin{align*}
        \E[|\vdot{X}{s}|^2 | \Eps^c] &\le \E[|\vdot{X}{s}|^{2p} | \Eps^c]^{1/p} \E[1 | \Eps^c]^{1/q},
    \end{align*}
    for any $p, q \ge 1$ such that $1/p + 1/q = 1$. We can take $p$ as arbitrary as we want, say $p = \log k$, in order to get rid of $k$ factor in equation \eqref{eq:aux_angle_2}. Then,
    \begin{align*}
        \E[|\vdot{X}{s}|^2 | \Eps^c] &\le \E[|\vdot{X}{s}|^{2p} | \Eps^c]^{1/p} \E[1 | \Eps^c]^{1/q} \le k^{1/p} (4^p \Gamma(1+p))^{1/p} \\
        &\le 4e (\Gamma(1+p)^{1/2p})^{2} \le C \log k,
    \end{align*}
    for some universal constant $C > 0$. We used the fact that $\Gamma(1+p) \le (p+1)^{p}$. The proof of Lemma \ref{lemma:threshold comparison} can be written similarly. 
\end{proof}

\begin{remark}
\label{remark:technical_lemmas}
These lemmas are modified from \cite{balakrishnan2017statistical, yi2016solving} to involve multiple components and higher order moments. They are also used in proofs of finite-sample EM, to find sub-exponential norm \cite{vershynin2010introduction} of random variables conditioned on specific events, as boundedness of any $p^{th}$ moment by Gamma function implies sub-Gaussianity. We conjecture that the $k$ factor in \eqref{eq:aux_angle_2} and \eqref{eq:aux_threshold_2} are sub-optimal, and it will improve the SNR condition by $O(\log k)$ if resolved.
\end{remark}

\subsection{Bounding $B$}
\label{Appendix:population_EM_B}
Since $\|B\| = \sup_{s \in \mathbb{S}^{d-1}} \E_D[w_1 \vdot{X}{s} (Y - \vdot{X}{\beta_1^*})]$, for any fixed unit vector $s$, we bound 
\begin{align*}
    B_s &:= |\E_D[w_1 \vdot{X}{s} (Y - \vdot{X}{\beta_1^*})]| \\
    &= |\E_D[w_1 \vdot{X}{s} (Y - \vdot{X}{\beta_1^*})] - \E_D[w_1^* \vdot{X}{s} (Y - \vdot{X}{\beta_1^*})]| \\
    &= |\E_D[\Delta_{w} \vdot{X}{s} (Y - \vdot{X}{\beta_1^*})]| \\
    &\le \pi_1^* \underbrace{|\E_{\D_1}[\Delta_{w} \vdot{X}{s} (Y - \vdot{X}{\beta_1^*})]|}_{B_1} + \sum_{j\neq 1} \pi_j^* \underbrace{|\E_{\D_j}[\Delta_{w} \vdot{X}{s} (Y - \vdot{X}{\beta_1^*})]|}_{B_j}.
\end{align*}
We will then bound $B_1$ and $B_j$ separately, as $B_1$ is the error term from its own component and $B_j$ is the error from other components.

Term in $B_j$ can be decoupled as
\begin{align*}
    B_j &= |\E_{\D_j}[\Delta_{w} \vdot{X}{s} \vdot{X}{\beta_j^* - \beta_1^*}] + \E_{\D_j}[\Delta{w} \vdot{X}{s} e]| \\
    &\le \underbrace{|\E_{\D_j}[\Delta_{w} \vdot{X}{s} \vdot{X}{\beta_j^* - \beta_1^*}]|}_{b_1} + \underbrace{|\E_{\D_j}[\Delta{w} \vdot{X}{s} e]|}_{b_2}.
\end{align*}
Then for each $j = 1, ..., k$, we give a bound for $B_j$. We divide the cases between $\max_j \|\Delta_j\| > 1$ and $\max_j \|\Delta_j\| \le 1$. The proof for $\|\Delta_j\| \le 1$ will be given in Appendix \ref{appendix:population_EM_B_II}. We use $D_m$ to denote $\max_j \|\Delta_j\|$ to simplify the notations. We also define $\rho_{jl} := \pi_l^* / \pi_j^*$ for $j \neq l$. 

\paragraph{Case I. $\max_j \| \Delta_j \| > 1$:} 
\paragraph{\bm $j \neq 1:$} To bound first term, define four events as follows:
\begin{align*}
    & \Eps_1 = \{|\vdot{X}{\beta_j^* - \beta_1^*}| \ge 4\sqrt{2} \tau_j \} \\
    & \Eps_2 = \{4(|\vdot{X}{\Delta_j}| \vee |\vdot{X}{\Delta_1}|) \le |\vdot{X}{\beta_j^* - \beta_1^*}| \} \\
    & \Eps_3 = \{|e| \le \tau_j\} \\
    & \Eps = \Eps_1 \cap \Eps_2 \cap \Eps_3.
\end{align*}
When all four events happen at the same time, it is a good sample: weights given to this sample is almost 0, as it comes from component $j$. For other events, we bound the probability of each event with respect to $\Delta_{j}$ and $\tau_j$. We decide threshold parameter $\tau_j$ at the end of the stage.
\begin{align*}
    b_1 &\le |\E_{\D_j}[\Delta_{w} \vdot{X}{s} \vdot{X}{\beta_j^* - \beta_1^*} \indic_{\Eps}]| + |\E_{\D_j}[\Delta_{w} \vdot{X}{s} \vdot{X}{\beta_j^* - \beta_1^*} \indic_{\Eps_1^c \cap \Eps_2}]| \\
    &+ |\E_{\D_j}[\Delta_{w} \vdot{X}{s} \vdot{X}{\beta_j^* - \beta_1^*} \indic_{\Eps_2^c}]| + |\E_{\D_j}[\Delta_{w} \vdot{X}{s} \vdot{X}{\beta_j^* - \beta_1^*} \indic_{\Eps_3^c}]|.
\end{align*}

\begin{enumerate}
    \item Event $\Eps$: Observe the value of the weight $w_1$. First note that 
    \begin{align*}
        (\vdot{X}{\beta_j^* - \beta_j} + e)^2 &\le 2|\vdot{X}{\Delta_j}|^2  + 2e^2 \le |\vdot{X}{\beta_j^* - \beta_1^*}|^2/8 + 2e^2 \\
        (\vdot{X}{\beta_j^* - \beta_1} + e)^2 &\ge |\vdot{X}{\beta_j^* - \beta_1^*} - \vdot{X}{\Delta_1}|^2/2 - e^2 \ge (9/32)  |\vdot{X}{\beta_j^* - \beta_1^*}|^2 - e^2. 
    \end{align*}
    Then,
    \begin{align}
        \label{eq:w_1_bound}
        w_1 &\le \frac{\pi_1 \exp(-(Y - \vdot{X}{\beta_1})^2/2)}{\pi_1 \exp(-(Y - \vdot{X}{\beta_1})^2/2) + \pi_j \exp(-(Y - \vdot{X}{\beta_j})^2/2)}  \nonumber \\
        &= \frac{\pi_1 \exp(-(\vdot{X}{\beta_j^* - \beta_1} + e)^2/2)}{\pi_1 \exp(-(\vdot{X}{\beta_j^* - \beta_1} + e)^2/2) + \pi_j \exp(-(\vdot{X}{\beta_j^* - \beta_j} + e)^2/2)} \nonumber  \\
        &\le (\pi_1 / \pi_j) \exp \left(((\vdot{X}{\beta_j^* - \beta_j} + e)^2-(\vdot{X}{\beta_j^* - \beta_1} + e)^2) /2 \right) \nonumber  \\
        &\le (\pi_1 / \pi_j) \exp \left((-5|\vdot{X}{\beta_j^* - \beta_1^*}|^2/32 + 3e^2)/2 \right) \nonumber  \\
        &\le (\pi_1 / \pi_j) \exp(-\tau_j^2).
    \end{align}
    Similarly, we get
    \begin{align*}
        w_1^* &\le (\pi_1^* / \pi_j^*) \exp \left((e^2-(\vdot{X}{\beta_j^* - \beta_1^*} + e)^2) /2 \right) \\
        &\le (\pi_1^* / \pi_j^*) \exp \left(\left(e^2 - (|\vdot{X}{\beta_j^* - \beta_1^*}| - |e|)^2 \right) / 2 \right) \\
        &\le (\pi_1^* / \pi_j^*) \exp \left((\tau_j^2 - 16\tau_j^2)/2 \right) \\
        &\le (\pi_1^* / \pi_j^*) \exp \left(- \tau_j^2 \right).
    \end{align*}
    
    Note that due to our initialization condition for $\pi_j$ for all $j$, $\rho_{j1} = \pi_1^* / \pi_j^* \le 3 \pi_1 / \pi_j$.
    
    Thus, $|\Delta_w| \le 3 \rho_{j1} \exp\left( -\tau_j^2 \right)$. From this inequality, we can get
    \begin{align*}
        |\E_{\D_j}[\Delta_{w} \vdot{X}{s} \vdot{X}{\beta_j^* - \beta_1^*} \indic_{\Eps}]| &\le 3 \rho_{j1} \exp \left(-\tau_j^2\right) \E_{\D_j}[|\vdot{X}{s} \vdot{X}{\beta_j^* - \beta_1^*}|] \\
        &\le 3 \rho_{j1} \exp \left(-\tau_j^2 \right) R_{j1}^*,
    \end{align*}
    where the last inequality comes from Cauchy-Schwartz inequality.

    \item Event $\Eps_1^c \cap \Eps_2$: In this case, from Lemma \ref{lemma:threshold comparison}, 
    $$
        P(\Eps_1^c \cap \Eps_2) \le P(\Eps_1^c) \le \frac{4\sqrt{2} \tau_j}{\| \beta_j^* - \beta_1^* \|}.
    $$
    Then, we proceed as
    \begin{align*}
        |\E_{\D_j}[\Delta_{w} \vdot{X}{s} \vdot{X}{\beta_j^* - \beta_1^*} \indic_{\Eps_1^c \cap \Eps_2}]| &\le 4\sqrt{2} \tau_j \E_{\D_j}[|\Delta_{w} \vdot{X}{s} \indic_{\Eps_1^c \cap \Eps_2}|] \\
        &\le 4\sqrt{2} \tau_j \E_{\D_j}[|\Delta_{w} \vdot{X}{s} \indic_{\Eps_1^c}|] \\
        &\le 4\sqrt{2} \tau_j \sqrt{\E[\Delta_w^2|\Eps_1^c]} \sqrt{\E[\vdot{X}{s}^2|\Eps_1^c]} P(\Eps_1^c) \\
        &\le 4\sqrt{2} \tau_j P(\Eps_1^c) \le \frac{32 \tau_j^2}{R_{j1}^*}.
    \end{align*}
    
    \item Event $\Eps_2^c$: Bound it as follows:
    \begin{align*}
        |\E_{\D_j}[\Delta_{w} \vdot{X}{s} \vdot{X}{\beta_j^* - \beta_1^*} &\indic_{\Eps_2^c}]| \le \sqrt{\E[\Delta_w^2 \vdot{X}{s}^2 | \Eps_2^c]} \sqrt{\E[\vdot{X}{\beta_j^* - \beta_1^*}^2| \Eps_2^c]} P(\Eps_2^c).
    \end{align*}
    Under this event, we note that $$\vdot{X}{\beta_j^* - \beta_1^*} \le 4(|\vdot{X}{\Delta_j}| \vee |\vdot{X}{\Delta_1}|) \le 4(|\vdot{X}{\Delta_j}| + |\vdot{X}{\Delta_1}|).$$
    \begin{align*}
        \E[\vdot{X}{\beta_j^* - \beta_1}^2| \Eps_2^c] &\le \E[32|\vdot{X}{\Delta_j}|^2 + 32|\vdot{X}{\Delta_1}|^2 | \Eps_2^c] \\ 
        &\le 32 (\E[|\vdot{X}{\Delta_j}|^2| \Eps_2^c] + \E[|\vdot{X}{\Delta_1}|^2| \Eps_2^c]) \\
        &\le 512 D_m^2,
    \end{align*}
    where we used Lemma \ref{lemma:angle_comparison} for bounding $\E[\vdot{X}{\Delta_j}^2|\Eps_2^c]$. 
    
    Now plugging this into the above, 
    \begin{align*}
        \sqrt{\E[\vdot{X}{s}^2 | \Eps_2^c]} &\sqrt{\E[\vdot{X}{\beta_j^* - \beta_1^*}^2 | \Eps_2^c]} P(\Eps_2^c) \\
        &\le 64 D_m P(\Eps_2^c) \le 512 D_m \frac{D_m}{R_{j1}^*}.
    \end{align*}

    \item Event $\Eps_3^c$: Similarly, 
    \begin{align*}
        |\E_{\D_j}[\Delta_{w} \vdot{X}{s} \vdot{X}{\beta_j^* - \beta_1^*} \indic_{\Eps_3^c}]| &\le \sqrt{\E[\Delta_w^2 \vdot{X}{s}^2 | \Eps_3^c]} \sqrt{\E[\vdot{X}{\beta_j^* - \beta_1^*}^2| \Eps_3^c]} P(\Eps_3^c) \\
        &\le \| \beta_j^* - \beta_1^* \| P(\Eps_3^c) \\
        &\le 2 R_{j1}^* \exp(-\tau_j^2/2) \\
        &\le 2 R_{j1}^* \exp(-\tau_j^2/2) D_m.
    \end{align*}
\end{enumerate}
We used independence of $e$ and $X$. Combining all, 
\begin{align}
    \label{eq:D_m_ge_1_b_1}
    b_1 \le O(\exp(-\tau_j^2/2) (1 \vee \rho_{j1}) R_{j1}^* + \tau_j^2 / R_{j1}^* + D_m /R_{j1}^*) D_m.
\end{align} 

Now we turn our attention to $b_2$. Recall $b_2 = |\E_{\D_j} [\Delta_w \vdot{X}{s} e]|$. For this setup, 
\begin{align*}
    b_2 &\le |\E_{\D_j} [\Delta_w \vdot{X}{s} e \indic_{\Eps}]| + |\E_{\D_j} [\Delta_w \vdot{X}{s} e \indic_{\Eps_1^c}]| \\
    &+ |\E_{\D_j} [\Delta_w \vdot{X}{s} e \indic_{\Eps_{2^c}}]| + |\E_{\D_j} [\Delta_w \vdot{X}{s} e \indic_{\Eps_3^c}]|.
\end{align*}
Under good event $\Eps$, as previously we have $|\Delta_w| \le 3 \rho_{j1} \exp(-\tau_j^2)$, thus
$$
    |\E_{\D_j} [\Delta_w \vdot{X}{s} e \indic_{\Eps}]| \le 3 \rho_{j1} \exp(-\tau_j^2) \E_{\D_j} [|\vdot{X}{s}e|] \le 3 \rho_{j1} \exp(-\tau_j^2).
$$

Similarly, we go through on the bad events. First,
\begin{align*}
    |\E_{\D_j} [\Delta_w \vdot{X}{s} e \indic_{\Eps_1^c}]| &\le \sqrt{\E_{\D_j} [\vdot{X}{s}^2|\Eps_1^c]} \sqrt{\E_{\D_j} [e^2|\Eps_1^c]} P(\Eps_1^c) \le c_1 \tau_j / R_{j1}^*,
\end{align*}
where we used Lemma \ref{lemma:threshold comparison} for bounding $\E_{\D_j} [\vdot{X}{s}^2|\Eps_1^c]$.

Second, 
\begin{align*}
    |\E_{\D_j} [\Delta_w \vdot{X}{s} e \indic_{\Eps_2^c}]| &\le \sqrt{\E_{\D_j} [\vdot{X}{s}^2|\Eps_2^c]} \sqrt{\E_{\D_j} [e^2|\Eps_2^c]} P(\Eps_2^c) \le c_2 D_m / R_{j1}^*.
\end{align*}
where we used Lemma \ref{lemma:angle_comparison} for bounding $\E_{\D_j} [\vdot{X}{s}^2|\Eps_2^c]$.

Finally,
\begin{align*}
    |\E_{\D_j} [\Delta_w \vdot{X}{s} e \indic_{\Eps_3^c}]| &\le \sqrt{\E_{\D_j} [\vdot{X}{s}^2 e^2]} \sqrt{ P(\Eps_3^c)} \le c_3 \exp(-\tau_j^2/4).
\end{align*}

Combining three items, we have
\begin{align}
    \label{eq:D_m_ge_1_b_2}
    b_2 \le O ((1 \vee \rho_{j1}) \exp(-\tau_j^2/4) + \tau_j/R_{j1}^* + D_m/R_{j1}^*).
\end{align}

Now we set
$$
\tau_j = c_\tau \sqrt{\log (R_{j1}^*k / (1 \wedge \rho_{j1}))}, \ R_{j1}^* > c_r k \rho_{j1}^{-1} \log(R_{j1}^*).
$$ 
With given good initialization $D_m / R_{j1}^* \le c_D \rho_{j1} / k$, we get $b_1 < c_b D_m \rho_{j1} / k$ and $b_2 \le c_{b'} D_m \rho_{j1} / k$ since $D_m \ge 1$. Combining \eqref{eq:D_m_ge_1_b_1} and \eqref{eq:D_m_ge_1_b_2}, we get $B_j \le c_B D_m \rho_{j1} / k$ for some small universal constant $c_B < 1/4$ with large enough $c_\tau, c_r$ and small enough $c_D$. 
\\

\paragraph{$j = 1$}: We only need to consider bounding $b_2 = |\E_{\D_j} [\Delta_w \vdot{X}{s} e]|$. We define some events similarly, but each involves multiple factors in this case.
\begin{align*}
    & \Eps_1 = \{|\vdot{X}{\beta_1^* - \beta_j}| \ge 4 \tau, \ \forall j \neq 1 \} \\
    & \Eps_2 = \{4|\vdot{X}{\Delta_1}| \le |\vdot{X}{\beta_1^* - \beta_j}|, \ \forall j \neq 1 \}\\
    & \Eps_3 =  \{|e| \le \tau \}, \\
    & \Eps = \Eps_1 \cap \Eps_2 \cap \Eps_3.
\end{align*}
Then follow the same path as in cases $j \neq 1$,
\begin{align*}
    b_2 &\le |\E_{\D_j} [\Delta_w \vdot{X}{s} e \indic_{\Eps}]| + |\E_{\D_j} [\Delta_w \vdot{X}{s} e \indic_{\Eps_1^c}]| \\
    &+ |\E_{\D_j} [\Delta_w \vdot{X}{s} e \indic_{\Eps_{2^c}}]| + |\E_{\D_j} [\Delta_w \vdot{X}{s} e \indic_{\Eps_3^c}]|.
\end{align*}

Then, on event $\Eps_1 \cap \Eps_2 \cap \Eps_3$, for all $j \neq 1$, we have 
\begin{align*}
w_j &\le (\pi_1 / \pi_j) \exp\left( (-(\vdot{X}{\beta_1^* - \beta_j} + e)^2 + (\vdot{X}{\Delta_1} + e)^2))/2 \right) \le 3 \rho_{j1} \exp (-3\tau^2/2),
\end{align*}
as before. Thus, $w_1 \ge 1 - 3 k \rho_\pi \exp (-3\tau^2/2)$. Similarly, $w_1^* \ge 1 - 3 k \rho_\pi \exp(-3\tau^2/2)$. Thus, $\Delta_w$ can be at most $k 3 \rho_\pi \exp(-3\tau^2/2)$. Then, 
$$
    |\E_{\D_1} [\Delta_w \vdot{X}{s} e \indic_{\Eps}]| \le 3 k \rho_\pi \exp(-3\tau^2/2) \E_{\D_1} [|\vdot{X}{s}e|] \le 3 k \rho_\pi \exp(-3 \tau^2 / 2).
$$

We can go over other events similarly.
\begin{align*}
    |\E_{\D_1} [\Delta_w \vdot{X}{s} e \indic_{\Eps_1^c}]| &\le \sqrt{\E_{\D_1} [\vdot{X}{s}^2|\Eps_1^c]} \sqrt{\E_{\D_1}[e^2|\Eps_1^c]} P(\Eps_1^c) \le c_1 \sqrt{\log k} \frac{k\tau}{R_{min}}.
\end{align*}
\begin{align*}
    |\E_{\D_1} [\Delta_w \vdot{X}{s} e \indic_{\Eps_2^c}]| &\le \sqrt{\E_{\D_1} [\vdot{X}{s}^2|\Eps_2^c]} \sqrt{\E_{\D_1}[e^2 | \Eps_2^c]} P(\Eps_2^c) \le c_2 \sqrt{\log k} \frac{kD_m}{R_{min}}.
\end{align*}
\begin{align*}
    |\E_{\D_1} [\Delta_w \vdot{X}{s} e \indic_{\Eps_3^c}]| &\le \sqrt{\E_{\D_1} [\vdot{X}{s}^2 e^2]} \sqrt{ P(\Eps_3^c)} \le c_3 \exp(-\tau^2/4).
\end{align*}
For first two inequalites, we used Lemma \ref{lemma:angle_comparison} and \ref{lemma:threshold comparison}. They all gives a bound for $b_2$ as,
\begin{equation}
    \label{eq:b2_Dm_greater_1_j_eq_1}
    b_2 \le O(k \rho_\pi \exp(-\tau^2/4) + (k \sqrt{\log k}) \tau / R_{min} + (k \sqrt{\log k}) D_m/R_{min}).
\end{equation}
Now we set $\tau = \Theta(\sqrt{\log (k\rho_\pi)})$, $R_{min} = \Omega(k \log(k\rho_\pi))$ and $D_m = O(R_{min}/(k\sqrt{\log k}))$, and we get $b_2 \le c_B$ and $B_1 = b_2 \le c_B D_m$.

Combining \eqref{eq:D_m_ge_1_b_1}, \eqref{eq:D_m_ge_1_b_2}, and \eqref{eq:b2_Dm_greater_1_j_eq_1}, we get the first part of Lemma \ref{lemma:population_bound_b}. We conclude
$$
    B = \pi_1^* B_1 + \sum_j \pi_j^* B_j \le \pi_1^* c_B D_m + \sum_{j \neq 1} \pi_j^* c_B D_m \rho_{j1} / k = \pi_1^* \left(c_B + \sum_{j \neq 1} c_B / k \right) D_m \le 2 \pi_1^* c_B D_m,
$$
where we used $\pi_j^* \rho_{j1} = \pi_1^*$. Thus $B \le c_B' \pi_1^* D_m$ for some universal constant $c_B' \in (0, 1/4)$ with properly set constants in the proof.

\paragraph{Update for mixing weights.} In this case $D_m \ge 1$, we will not focusing on improvement over the quality of $\pi_j$. Instead, we will only show that $\pi_j$ stays in a neighborhood of the true parameter, {\it i.e.}, $|\pi_j - \pi_j^*| \le \pi_j^* / 2$. It can be actually very easily shown with reusing the results we derived for $\beta$. Observe that
\begin{align}
    \label{eq:weight_update}
    \pi_1^+ - \pi_1^* = \E_\D [w_1 - w_1^*] = \E_\D [\Delta_w].
\end{align}
Now we can proceed as before:
\begin{align*}
    \E_\D [\Delta_w] &= \pi_1^* \E_{\D_1} [\Delta_w] + \sum_{j\neq1} \pi_j^* \E_{\D_j} [\Delta_w]
    \le \pi_1^* \underbrace{|\E_{\D_1} [\Delta_w]|}_{P_1} + \sum_{j\neq1} \pi_j^* \underbrace{|\E_{\D_j} [\Delta_w]|}_{P_j}.
\end{align*}
Moving along the same trajectory as in \eqref{eq:D_m_ge_1_b_2} for $j \neq 1$ case, 
\begin{align*}
    P_j \le O \left((1 + \pi_1^*/\pi_j^*) \exp(-\tau_j) + \tau_j / R_{j1}^* + D_m / R_{j1}^* \right).
\end{align*}
With properly setting parameters similarly as in $\beta$ case, we get $P_j \le \rho_{j1}/ 4k$. For $j=1$ case, in fact, we can reuse the result for \eqref{eq:b2_Dm_greater_1_j_eq_1} as it is. To see this, for instance, 
\begin{align*}
    |\E_{\D_1} [\Delta_w \indic_{\Eps_1^c}]| \le \E_{\D_1} [|\indic_{\Eps_1^c}|] = P(\Eps_1^c) \le c_1 k\tau /R_{min}.
\end{align*}
We can do for all cases similarly to get
\begin{align*}
    P_1 \le O \left(k \rho_\pi \exp(-\tau^2) + k \tau / R_{min} + k D_m / R_{min} \right).
\end{align*}
By setting the parameters similarly as before, {\it i.e.}, $\tau = \Theta(\sqrt{\log (k \rho_\pi)}), R_{min} = \tilde{\Omega}(k), D_m = O (R_{min} / k)$, we can get $P_1 \le 1/4$ with properly set constants. Therefore, $|\pi_1^+ - \pi_1^*| \le \pi_1^* / 2$ as desired.

\subsection{Bounding $A$}
\label{Appendix:population_EM_A}

We start it with a following observation.
\begin{align*}
    \E_\D[w_1 X X^{\top}] &\succeq \pi_1 \E_{\D_1} [w_1 XX^{\top}].
\end{align*}
Thus, we will only focus on giving a constant lower bound for $\E_{\D_1}[w_1 XX^{\top}]$. We define good events as
\begin{align*}
    \Eps_1 &= \{|e| \le \tau \} \\
    \Eps_2 &= \{|\vdot{X}{\beta_j - \beta_1^*}| \ge 4|\vdot{X}{\Delta_1}|, \forall j \neq 1 \} \\
    \Eps_3 &= \{|\vdot{X}{\beta_j - \beta_1^*}| \ge 4\tau, \forall j \neq 1 \}.
\end{align*}
We will set $\tau = c_\tau \sqrt{\log (k \rho_\pi)}$ with some large constant $c_\tau > 0$ in this case. Let $\Eps = \Eps_1 \cap \Eps_2 \cap \Eps_3$.

Using 
$$\E_{\D_1}[w_1 XX^{\top}] = E_{\D_1}[XX^{\top}] - \underbrace{E_{\D_1}[(1-w_1) XX^{\top} \indic_{\Eps}]}_{(i)} - \underbrace{E_{\D_1}[(1 - w_1) XX^{\top} \indic_{\Eps^c}]}_{(ii)},$$ we will give an upper bound to last two terms.

Under $\Eps$, it can be similarly shown as before that $(1 - w_1) \le 3 k \rho_\pi \exp(-\tau^2)$. Thus, (i) is easily bounded:
$$
    \E_{\D_1}[(1-w_1) XX^{\top} \indic_{\Eps}] \preceq 3 k \rho_\pi \exp(-\tau^2) \E_{\D_1}[XX^{\top} \indic_{\Eps}] \preceq 3 k \rho_\pi \exp(-\tau^2) I.
$$

We should split the cases for (ii). Observe that
\begin{align*}
    \E_{\D_1}[(1 - w_1) XX^{\top} \indic_{\Eps^c}] &\preceq \E_{\D_1}[XX^{\top} \indic_{\Eps^c}] \\
    &\preceq \E_{\D_1}[XX^{\top} | \Eps_1^c] P(\Eps_1^c) + \E_{\D_1}[XX^{\top} | \Eps_2^c] P(\Eps_2^c) + \E_{\D_1}[XX^{\top} | \Eps_3^c] P(\Eps_3^c)
\end{align*}
We bound each one by one. First, 
\begin{align*}
    \E_{\D_1}[XX^{\top} | \Eps_1^c] P(\Eps_1^c) &= \E_{\D_1}[XX^{\top} | |e| \ge \tau] P(e \ge \tau) \\
    &= \E_{\D_1} [XX^{\top}] P(e \ge \tau) \preceq \exp(-\tau^2/2) I.
\end{align*}
where in the first inequality we used independence of $e$ and $X$.

For the second term, $$\E_{\D_1} [XX^{\top}|\Eps_2^c] \preceq c_1 (\log k) I,$$ from Corollaray \ref{corollary:key_bound}. Meanwhile, we have $P(\Eps_2^c) \le k \frac{4 \|\Delta_1\|}{R_{min}}$. Thus, 
$$
    \E_{\D_1}[XX^{\top} | \Eps_2^c] P(\Eps_2^c) \preceq c_2 (k \log k) \frac{D_m}{R_{min}} I.
$$

Finally, we bound the operator norm for 
$$
    \E_{\D_1}[XX^{\top}|\Eps_3^c] = \E_{\D_1}[XX^{\top}|\exists j \neq 1, \vdot{X}{\beta_j - \beta_1^*} \le 4\tau] \preceq c_3 (\log k) I,
$$
from Corollary \ref{corollary:key_bound}. On one hand, $P(\Eps_3^c) \le k \frac{4\tau}{R_{min}}$. Now combining three pieces, we have
$$
    \|(ii)\|_{op} \le \exp(-\tau^2/2) + c_4 (k \log k) \frac{D_m}{R_{min}} + c_5 (k \log k) \frac{\tau}{R_{min}}.
$$

Return to bounding $\E_{\D_1}[w_1 XX^{\top}] = I - (i) - (ii)$, we have
\begin{align*}
    \E_{\D_1}[w_1 XX^{\top}] \succeq 1 - O\left(k \rho_\pi \exp(-\tau^2/2) + (k\log k) \frac{D_m}{R_{min}} + (k\log k) \frac{\tau}{R_{min}} \right).
\end{align*}
Giving appropriate $\tau = c_\tau \sqrt{\log k \rho_\pi}$, $D_m/R_{min} \le 1/\tilde{O}(k)$, $R_{min} = \tilde{\Omega}(k)$, we have $\| \E_{\D_1}[w_1 XX^{\top}] \|_{op} \ge 1/2$. Thus, $\|A^{-1}\|_{op} \le 2/\pi_1^*$.

\section{Proofs for Finite-Sample EM}
\label{appendix:concentration_finite_sample}

\subsection{Proofs for concentration of B}
\label{appendix:concentrate_B}
We will consider the following events for concentration result
\begin{align*}
    & \Eps_{j} = \{ \text{sample comes from $j^{th}$ linear model} \} \\
    & \Eps_{j,1} = \{ |e| \le \tau_j \}, \\
    & \Eps_{j,2} = \{ 4(|\vdot{X}{\Delta_1}| \vee |\vdot{X}{\Delta_j}|) \le |\vdot{X}{\beta_j^* - \beta_1^*}|\}, \\
    & \Eps_{j,3} = \{ |\vdot{X}{\beta_j^* - \beta_1^*}| \ge 4\sqrt{2} \tau_j\} \\
    & \Eps_{j,good} = \Eps_{j,1} \cap \Eps_{j,2} \cap \Eps_{j,3}
\end{align*}
then decompose each sample using the indicator functions of these events. 
\begin{align*}
    w_{1,i} X_i (y_i - \vdot{X}{\beta_1^*}) = \Bigg (&\sum_{j\neq 1}^k  w_{1,i} X_i \vdot{X_i}{\beta_j^* - \beta_1^*} \indic_{\Eps_{j} \cap \Eps_{j,good}} + w_{1,i} X_i \vdot{X_i}{\beta_j^* - \beta_1^*} \indic_{\Eps_{j} \cap \Eps_{j,1}^c} \\
    &+ w_{1,i} X_i \vdot{X_i}{\beta_j^* - \beta_1^*} \indic_{\Eps_{j} \cap \Eps_{j,1} \cap \Eps_{j,2}^c} + w_{1,i}X_i\vdot{X_i}{\beta_j^* - \beta_1^*} \indic_{\Eps_{j} \cap \Eps_{j,1} \cap \Eps_{j,2} \cap \Eps_{j,3}^c} \\
    &+ w_{1,i} X_i e_i \indic_{\Eps_{j} \cap \Eps_{j,good}}
    + w_{1,i} X_i e_i \indic_{\Eps_{j} \cap \Eps_{j,1}^c} + w_{1,i} X_i e_i \indic_{\Eps_{j} \cap \Eps_{j,1} \cap (\Eps_{j,2} \cup \Eps_{j,3})^c} \Bigg) \\
    &+ w_{1,i} X_i e_i \indic_{\Eps_1}.
\end{align*}
We will bound the deviation under each event separately. Before getting into detailed analysis, we remind some basics on sub-exponential random variables.

From standard tail bound for sub-exponential random variable $W$ with sub-exponential norm $K$, we have \cite{vershynin2010introduction}
$$
    P\left( \left|\frac{1}{n} \sum_i W_i - \E[W] \right| \ge t \right) \le 2 \exp\left(-Cn \min(t/K, (t/K)^2) \right).
$$
If $W$ is a random vector in $\mathbb{R}^d$ with all elements being sub-exponential with same norm $K$, then
\begin{align}
    \label{eq:sub_exp_bound}
    P\left( \left\|\frac{1}{n} \sum_i W_i - \E[W] \right\| \ge t \right) &\le \sum_{j=1}^d 2 P\left( \left |\frac{1}{n} \sum_i (W_i)_j - \E[(W)_j] \right| \ge t/\sqrt{d} \right) \nonumber \\
    &\le 2 d \exp\left(-Cn \min \left( \frac{t}{K\sqrt{d}} , \left(\frac{t}{K\sqrt{d}} \right)^2 \right) \right) \nonumber \\
    &= \exp\left(-Cn \min \left( \frac{t}{K\sqrt{d}} , \left(\frac{t}{K\sqrt{d}} \right)^2 \right) + C' \log d \right).
\end{align}
Therefore, in order to achieve $\delta$ probability error bound, we should have
\begin{equation}
    t = O\left( K\sqrt{d} \left(\frac{\log(d/\delta)}{n} \vee \sqrt{\frac{\log(d/\delta)}{n}} \right) \right).
\end{equation}

Now we get into concentration of random variables multiplied with indicator functions. For each decomposed random variable, we will find the bound for deviations of empirical mean from true mean that holds with probability at least $1 - \delta/k^2 T$. 

\begin{enumerate}
    \item $w_{i,1} X_i \vdot{X_i}{\beta_j^* - \beta_1^*} \indic_{\Eps_{j} \cap \Eps_{j,good}}$: We first check if the target random variable is sub-exponential random vector. For any fixed direction $s \in \mathbb{S}^{d-1}$, we will show $W_i = w_{1,i} \vdot{X_i}{s} \vdot{X_i}{\beta_j^* - \beta_1^*} | \Eps_j \cap \Eps_{j,good}$ is sub-exponential by bounding sub-exponential norm.
    \begin{align*}
        \| W \|_{\psi_1} &= \sup_{p \ge 1} p^{-1} \E_\D[|w_1 \vdot{X}{s} \vdot{X}{\beta_j^* - \beta_1^*}|^p | \Eps_{j} \cap \Eps_{j,good} ]^{1/p} \\
        &= \sup_{p \ge 1} p^{-1} \E_{\D_j} [|w_1 \vdot{X}{s} \vdot{X}{\beta_j^* - \beta_1^*}|^p | \Eps_{j,good}]^{1/p}.
    \end{align*}
    Now recall \eqref{eq:w_1_bound} that how we bounded $w_1$. Under good event, we know that $w_1$ is less than 1 or $3\rho_{j1} \exp \left( (-5/32 \vdot{X}{\beta_j^* - \beta_1^*}^2 + 3e^2)/2 \right) \le 3\rho_{j1} \exp(-\tau_j^2)$. Thus, 
    \begin{align*}
        \| W \|_{\psi_1} &= 3\rho_{j1} \exp(-\tau_j^2) \sup_{p \ge 1} p^{-1} \E_{\D_j} [|\vdot{X}{s} \vdot{X}{\beta_j^* - \beta_1^*}|^p | \Eps_{j,good}]^{1/p} \\
        &\le \frac{3}{P(\Eps_{j,good} | \Eps_j)} \rho_{j1} \exp(-\tau_j^2) \sup_{p \ge 1} p^{-1} \E_{\D_j} [|\vdot{X}{s}|^{2p}]^{1/2p} \E_{\D_j} [|\vdot{X}{\beta_j^* - \beta_1^*}|^{2p}]^{1/2p} \\
        &\le C \rho_{j1} \exp(-\tau_j^2) R_{j1}^*,
    \end{align*}
    with sufficiently large constant $C>0$. We used the fact that $P(\Eps_{j,good} | \Eps_{j}) \ge 1/2$ given good enough initialization and SNR, and $l_p$-norm of Gaussian is bounded by $O(\sqrt{p})$. 
    
    Now we got a sub-exponential norm of $W$, we are almost ready to apply our Proposition \ref{lemma:indic_prob_decompose}. In order to invoke Proposition \ref{lemma:indic_prob_decompose}, we need to choose proper $n_e$. First let us bound the probability of large deviation of Bernoulli random variables $Z_i = \indic_{(X_i, y_i) \in \Eps_j \cap \Eps_{j,good}}$. Note that Bernstein's inequality for Bernoulli random variable is 
    \begin{align}
        \label{eq:Bernstein}
        P \left(|\frac{1}{n} \sum_{i=1}^{n} Z_i - E[Z]| \ge t \right) \le \exp\left(-\frac{n t^2}{2t + 2p/3} \right),
    \end{align}
    
    Observe that $p := P(\Eps_j \cap \Eps_{j,good}) \le P(\Eps_j) = \pi_j^*$. We can choose $n_e$ by checking if the following holds:
    \begin{align*}
        P(\sum_i Z_i \ge n_e + 1) \le P \left(\frac{1}{n} \sum_i Z_i - p \ge t \right) \le \delta / (k^2T).
    \end{align*}
    By solving the equation $\eqref{eq:Bernstein} = \delta/(k^2T)$, we get 
    $$t = O\left(\frac{\log(k^2T/\delta)}{n} + \sqrt{\frac{p \log(k^2T/\delta)}{n}} \right).$$
    Therefore, right choice of $n_e = np + O \left(\log(k^2T/\delta) \vee \sqrt{np \log(k^2T/\delta)} \right)$.
    
    We can also use Bernstein's inequality to get
    \begin{align}
        \label{ineq:prob_t2}
        P \left(\|\E[W]\| |\frac{1}{n} \sum_{i=1}^{n} Z_i - E[Z]| \ge t_2 \right) \le \exp\left(-\frac{n t_2^2}{(2t_2 + 3p) \|W\|_{\psi_1}^2} \right),
    \end{align}
    where we used basic fact that $\| E[W] \| \le \| W\|_{\psi_1}$ from \cite{vershynin2010introduction}. 
    For $t_2$, we set 
    \begin{align}
        \label{eq:set_t2}
        t_2 &= \|W\|_{\psi_1} O\left( \frac{\log(k^2T/\delta)}{n} \vee \sqrt{\frac{p}{n} \log(k^2T/\delta)} \right) \nonumber \\
        &= \|W\|_{\psi_1} O \left( \sqrt{p \vee \frac{1}{n}} \sqrt{\frac{\log^2 (k^2T/\delta)}{n}} \right).
    \end{align}
    
    Then recall \eqref{eq:sub_exp_bound}, we have
    \begin{align}
        \label{ineq:prob_t1}
        P \left(\|\frac{1}{n_e} \sum_{i=1}^{n_e} W_i - E[W]\| \ge \frac{n}{n_e}t_1\right) &\le \exp\left(-C n_e \min \left(\frac{n^2 t_1^2}{n_e^2 d \|W\|_{\psi_1}^2}, \frac{n t_1}{n_e \sqrt{d} \|W\|_{\psi_1}^2} \right) + C' \log d \right) \nonumber \\
        &= \exp\left(-C \min \left(\frac{n^2 t_1^2}{n_e d \|W\|_{\psi_1}^2}, \frac{n t_1}{\sqrt{d} \|W\|_{\psi_1}^2} \right) + C' \log d \right),
    \end{align}
    Therefore, proper scale of $t_1$ is
    \begin{align}
        \label{eq:set_t1}
        t_1 &= O \left(\|W\|_{\psi_1} \sqrt{\frac{n_e}{n}} \sqrt{\frac{d}{n} \log(d k^2T/\delta)} \vee \|W\|_{\psi_1} \frac{\sqrt{d} \log (dk^2T/\delta)}{n} \right) \nonumber \\
        &\le \|W\|_{\psi_1} O\left( \sqrt{p \vee \frac{1}{n}} \sqrt{\frac{d}{n} \log^2 (d k^2T/\delta)} \right).
    \end{align}
    Since $n = \tilde{\Omega}(1/\pi_{min})$ as we will use, with probability at least $1 - 3\delta / (k^2T)$, 
    \begin{align*}
        \| \frac{1}{n} \sum_i w_{1,i} X_i \vdot{X_i}{\beta_j^* - \beta_1^*} \indic_{\Eps_{j} \cap \Eps_{j,good}} - E[w_{1,i} X_i \vdot{X_i}{\beta_j^* - \beta_1^*} \indic_{\Eps_{j} \cap \Eps_{j,good}}] \| \\ 
        = O \left( \rho_{j1} R_{j1}^* \exp(-\tau_j^2) \sqrt{\pi_j^*} \sqrt{\frac{d}{n} \log^2 (dk^2T/\delta)}\right)
    \end{align*}
    
    \item $w_{1,i} X_i \vdot{X_i}{\beta_j^* - \beta_1^*} \indic_{\Eps_{j} \cap \Eps_{j,1}^c}$: It corresponds to the case where the noise power is larger than $\tau_j$. The probability of this event is $p:= P(\Eps_{j} \cap \Eps_{j,1}^c) \le 2\pi_j^* \exp(-\tau_j^2/2)$. In this case, $W_i = w_{1,i} \vdot{X_i}{s} \vdot{X_i}{\beta_j^* - \beta_1^*} | \Eps_j \cap \Eps_{j,1}^c$ is bounded as
    \begin{align*}
        \| W \|_{\psi_1} &= \sup_{p \ge 1} p^{-1} \E_\D[|w_1 \vdot{X}{s} \vdot{X}{\beta_j^* - \beta_1^*}|^p | \Eps_{j} \cap \Eps_{j,1}^c ]^{1/p} \\
        &= \sup_{p \ge 1} p^{-1} \E_{\D_j} [|w_1 \vdot{X}{s} \vdot{X}{\beta_j^* - \beta_1^*}|^p | \Eps_{j,1}^c]^{1/p} \\
        &\le \sup_{p \ge 1} p^{-1} \E_{\D_j} [|\vdot{X}{s} \vdot{X}{\beta_j^* - \beta_1^*}|^p | \Eps_{j,1}^c]^{1/p} \\
        &= \sup_{p \ge 1} p^{-1} \E_{\D_j} [|\vdot{X}{s} \vdot{X}{\beta_j^* - \beta_1^*}|^p]^{1/p} \\
        &\le C R_{j1}^*,
    \end{align*}
    for some constant C, where the last equality comes from the fact that $X$ and $e$ are independent. While we want to reuse \eqref{eq:set_t2} and \eqref{eq:set_t1} to decide deviation of means under this event, we also need to cancel out the norm of $W = O(R_{j1}^*)$. We consider two cases: $1/n < p^{1/c}$ and $1/n > p^{1/c}$ for some number $c \ge 2$. 
    
    If $1/n < p^{1/c}$, then $\sqrt{p \vee 1/n} \le p^{1/2c} = 2\exp(-\tau_j^2/(4c))$. We can just plug in this bound into \eqref{eq:set_t2} and \eqref{eq:set_t1} to get the deviation
    \begin{align*}
        \| \frac{1}{n} \sum_i w_{1,i} X_i \vdot{X_i}{\beta_j^* - \beta_1^*} \indic_{\Eps_{j} \cap \Eps_{j,1}^c} - E[w_{1,i} X_i \vdot{X_i}{\beta_j^* - \beta_1^*} \indic_{\Eps_{j} \cap \Eps_{j,1}^c}] \| \\ 
        = O\left(R_{j1}^* \exp(-\tau_j^2/(4c)) \sqrt{\frac{d}{n} \log^2 (dk^2T/\delta)} \right),
    \end{align*}
    with probability at least $1 - \delta/(k^2T)$. 
    
    On the other side, if $1/n > p^{1/c}$, then we will set $n_e = 0$, {\it i.e.}, no sample fell into this event. This is true with probability $1-np = 1 - 1/n^{c - 1}$. The statistical error is thus just 
    $$\E[W] p = O(R_{j1}^* \exp(-\tau_j^2/2)) \le O(R_{j1}^* \exp(-\tau_j^2/4) / n).$$ 
    By setting $c = 4$ and $n \ge (k^2T/\delta)^{1/3}$, this will hold with probability at least $1 - \delta/(k^2 T)$. 
    \begin{remark}
        \label{remark:high_probability_on_n}
        The high probability result is usually given as $1 - \exp(-cn)$, but it is also enough to show that it holds with probability $1 - 1/n^c$ for some constant $c > 0$. The choice of 3 is arbitrary and we could have picked any other larger constant.
    \end{remark}

    \item $w_{1,i} X_i \vdot{X_i}{\beta_j^* - \beta_1^*} \indic_{\Eps_{j} \cap \Eps_{j,1} \cap \Eps_{j,2}^c}$: Under this event, we first note that $|w_{1,i} \vdot{X_i}{s}$ $\vdot{X_i}{\beta_j^* - \beta_1^*}| \le 4 |\vdot{X_i}{s}| (|\vdot{X_i}{\Delta_j}| + |\vdot{X_i}{\Delta_j}|)$. In turn, 
    \begin{align*}
        \|W\|_{\psi_1} &= \sup_{p \ge 1} p^{-1} \E_\D[|w_1 \vdot{X}{s} \vdot{X}{\beta_j^* - \beta_1^*}|^p | \Eps_{j} \cap \Eps_{j,1} \cap \Eps_{j,2}^c]^{1/p} \\
        &= \sup_{p \ge 1} p^{-1} \E_{\D_j} [|w_1 \vdot{X}{s} \vdot{X}{\beta_j^* - \beta_1^*}|^p | \Eps_{j,1} \cap \Eps_{j,2}^c]^{1/p} \\
        &\le \sup_{p \ge 1} p^{-1} \E_{\D_j} [|\vdot{X}{s} \vdot{X}{\beta_j^* - \beta_1^*}|^p| \Eps_{j,2}^c]^{1/p} \\
        &\le 4 \sup_{p \ge 1} p^{-1} \E_{\D_j} \left[|\vdot{X}{s} (\vdot{X}{\Delta_1} + \vdot{X}{\Delta_j})|^p | \ \Eps_{j,2}^c \right]^{1/p} \\
        &\underset{(i)}{\le} 4 \sup_{p \ge 1} p^{-1} \left (\sqrt{\E_{\D_j} [|\vdot{X}{s}|^{2p} | \Eps_{j,2}^c]} \sqrt{|\vdot{X}{\Delta_j}|^{2p}|\Eps_{j,2}^c]} \right)^{1/p} \\
        &+ 4 \sup_{p \ge 1} p^{-1} \left (\sqrt{\E_{\D_j} [|\vdot{X}{s}|^{2p} | \Eps_{j,2}^c]} \sqrt{|\vdot{X}{\Delta_1}|^{2p}|\Eps_{j,2}^c]} \right)^{1/p} \\
        &\underset{(ii)}{\le} c_2 D_m,
    \end{align*}
    where (i) we used Minkowski's inequality and Cauchy-Schwartz inequality, then (ii) we invoked Lemma \ref{lemma:angle_comparison}. Recall that $D_m = \max_j \|\Delta_j\|$. Thus $W = w_{1} X \vdot{X}{\beta_j^* - \beta_1^*} | (\Eps_{j} \cap \Eps_{j,2}^c)$ is sub-exponential with parameter at most $c_2 D_m$. We can also check that $p:= P(\Eps_j \cap \Eps_{j,2}^c) \le O(\pi_j^* D_m / R_{j1}^*)$. 
    
    We choose $n_e = np + O(\log(k^2T/\delta) \vee \sqrt{np\log(k^2T/\delta)})$ as before. Using \eqref{eq:set_t2} and \eqref{eq:set_t1}, 
    \begin{align}
        \label{eq:finite_t1_t2_D_m_setting}
        t_1 &= O\left( D_m \sqrt{p \vee \frac{\log(d k^2T/\delta)}{n}} \sqrt{\frac{d}{n} \log(dk^2T/\delta)} \right), \nonumber \\
        t_2 &= O\left(D_m \sqrt{\frac{p \log(k^2T /\delta)}{n}} \vee D_m \frac{\log(k^2T/\delta)}{n} \right).
    \end{align}
    
    We can see that $n  = \Omega(d k \log(dk^2T/\delta) / \pi_1^*)$ suffices to ensure $t_1, t_2 < O(D_m \pi_1^* / k)$ since $p \le O(\pi_j^*/(k\rho_\pi)) = O(\pi_1^* / k)$ by the initialization condition. Overall, $t_1 + t_2$ is bounded by
    $$
        O\left(D_m \sqrt{\frac{\pi_j^* D_m}{R_{j1}^*} \vee \frac{\log(dk^2T/\delta)}{n}} \sqrt{\frac{d}{n} \log(dk^2T / \delta)} \right).
    $$

    \item $w_{1,i} X_i \vdot{X_i}{\beta_j^* - \beta_1^*} \indic_{\Eps_{j} \cap \Eps_{j,1} \cap \Eps_{j,2} \cap \Eps_{j,3}^c}$: In this case, we define $W$ as
    $$W_i = w_{1,i} X_i \vdot{X_i}{\beta_j^* - \beta_1^*} \indic_{\Eps_{j,1} \cap \Eps_{j,2}} | (\Eps_{j} \cap \Eps_{j,3}^c).$$
    In other words, we are leaving some events in the indicator. Then, we can restart from bounding the sub-exponential norm of $W$.
    \begin{align*}
        \|W\|_{\psi_1} &= \sup_{p \ge 1} p^{-1} \E_{\D_j} [|w_1 \vdot{X}{s} \vdot{X}{\beta_j^* - \beta_1^*}|^p \indic_{\Eps_{j,1} \cap \Eps_{j,2}} | \cap \Eps_{j,3}^c]^{1/p} \\
        &\le \sup_{p \ge 1} p^{-1} \E_{X\sim \mathcal{N}(0,I)} [ \left( \E_{Y \sim \mathcal{N}(\vdot{X}{\beta_j^*},1)} [w_1] \right) |\vdot{X}{s} \vdot{X}{\beta_j^* - \beta_1^*}|^p \indic_{\Eps_{j,2}} | \Eps_{j,3}^c]^{1/p}.
    \end{align*}
    Then, use the following bound for expectation of $w_1$ when $\Eps_{j,2}$ is true,
    \begin{align}
        \label{eq:w_1_integrated_bound}
        \E_{Y \sim \mathcal{N}(\vdot{X}{\beta_j^*}, 1)}[w_1] &\le \E_{e \sim \mathcal{N}(0, 1)} \left[\min \left(3\rho_{j1} \exp \left( (-5/32 \vdot{X}{\beta_j^* - \beta_1^*}^2 + 3e^2 )/2 \right), 1 \right) \right] \nonumber \\
        &\le \E_{e \sim \mathcal{N}(0, 1)} \left[ 3\rho_{j1} \exp \left(-5/32 \vdot{X}{\beta_j^* - \beta_1^*}^2 + 3e^2 \right) \indic_{3e^2 \le |\vdot{X}{\beta_j^* - \beta_1^*}|^2/32}| \right] \nonumber \\
        &+ \E_{e \sim \mathcal{N}(0, 1)} \left[\indic_{3e^2 \ge \vdot{X}{\beta_j^* - \beta_1^*}^2/32} \right] \nonumber \\
        &\le \E_{e \sim \mathcal{N}(0, 1)} \Bigg[ 3\rho_{j1} \exp \left(-\vdot{X}{\beta_j^* - \beta_1^*}^2/16 \right) \Bigg] + P(3e^2 \ge |\vdot{X}{\beta_j^* - \beta_1^*}|^2/32) \nonumber \\
        &\le 5 (1 \vee \rho_{j1}) \exp(-\vdot{X}{\beta_j^* - \beta_1^*}^2 / 192).
    \end{align}
    
    Then we compute the upper bound for $\exp(-\vdot{X}{\beta_j^* - \beta_1^*}^2 / 192) |\vdot{X}{\beta_j^* - \beta_1^*}|^p$. Letting $|\vdot{X}{\beta_j^* - \beta_1^*}| = a$, and find a maximum for $-a^2 / 192 + p \log a$ by finding a zero point in its derivative. We get a maximum at $a = \sqrt{96p}$ with value $(96p)^{p/2} \exp(-p/2)$. Now plug this upper bound to continue bounding norm of $W$. 
    \begin{align*}
        \| W \|_{\psi_1} &\le \sup_{p \ge 1} p^{-1} \E_{X \sim \mathcal{N}(0,I)} \left[ \left(\E_{Y \sim \mathcal{N}(\vdot{X}{\beta_j^*}, 1)}[w_1]\right) \indic_{\Eps_{j,2}} |\vdot{X}{s} \vdot{X}{\beta_j^* - \beta_1^*}|^p | \Eps_{j,1}^c \right]^{1/p} \\
        &\le 5 (1 \vee \rho_{j1}) \sup_{p \ge 1} p^{-1} (96p)^{1/2} \exp(-1/2) \E_{X \sim \mathcal{N}(0,I)} \left[ |\vdot{X}{s}|^p | \Eps_{j,1}^c \right]^{1/p} \\
        &\le C (1 \vee \rho_{j1}),
    \end{align*}
    with sufficiently large constant $C>0$ and Lemma \ref{lemma:threshold comparison} for the final inequality. 
    
    The probability of this event $p := P(\Eps_{j} \cap \Eps_{j,2}^c) \le 4\sqrt{2} \pi_j^* \tau_j / R_{j1}^*$. Again we use Proposition \ref{lemma:indic_prob_decompose} to get a deviation of this random variable. We can set $t_1$ and $t_2$ as
    \begin{align*}
        t_1 &= O\left((1 \vee \rho_{j1}) \sqrt{p \vee \frac{\log(dk^2T /\delta)}{n}} \sqrt{\frac{d\log(dk^2T/\delta)}{n}} \right), \\
        t_2 &= O\left((1 \vee \rho_{j1}) \sqrt{\frac{p \log(k^2T /\delta)}{n}} \vee (1 \vee \rho_{j1}) \frac{\log(k^2T/\delta)}{n} \right).
    \end{align*}
    With probability at least $1 - \delta/k^2T$, we conclude that the deviation of sum under this event is
    $$
        O\left( (1 \vee \rho_{j1}) \left(\sqrt{\frac{\log(dk^2T/\delta)}{n}} \vee \sqrt{\frac{\pi_j^* \tau_j}{R_{j1}^*}}\right) \sqrt{\frac{d}{n} \log (dk^2T/\delta)} \right).
    $$
\end{enumerate}

Now we will bound terms for $w_{1,i} X_i e_i$, it is almost exact repetition of previous procedures when it comes from $j^{th} \neq1$ component. 
\begin{enumerate}
     \item $w_{i,1} X_i e_i \indic_{\Eps_{j} \cap \Eps_{j,good}}$, $j \neq 1$: One can show that following the exact same procedure for the first case we handled for $w_{i,1} X_i \vdot{X}{\beta_j^* - \beta_1^*}$. In this case, $W_i = w_{i,1} X_i e_i | \Eps_{j} \cap \Eps_{j,good}$, we can get $\|W\|_{\psi_1} \le C \rho_{j1} \exp(-\tau_j^2)$. To see this, 
     \begin{align*}
        \| W \|_{\psi_1} &= \sup_{p \ge 1} p^{-1} \E_\D[|w_1 \vdot{X}{s} e|^p | \Eps_{j} \cap \Eps_{j,good} ]^{1/p} \\
        &= \sup_{p \ge 1} p^{-1} \E_{\D_j} [|w_1 \vdot{X}{s} \vdot{X}{\beta_j^* - \beta_1^*}|^p \Eps_{j,good}]^{1/p} \\
        &\le 3/P(\Eps_{j,good} | \Eps_j) \rho_{j1} \exp(-\tau_j^2) \sup_{p \ge 1} p^{-1} \E_{\D_j} [|\vdot{X}{s} e|^p]^{1/p} \\
        &\le C \rho_{j1} \exp(-\tau_j^2). 
    \end{align*}
    
     Following the same trick we used with Proposition \ref{lemma:indic_prob_decompose}, (see \eqref{eq:set_t2} and \eqref{eq:set_t1}), we get
     \begin{align*}
        \| \frac{1}{n} \sum_i w_{1,i} X_i e_i \indic_{\Eps_{j} \cap \Eps_{j,good}} - \E[w_{1} X e \indic_{\Eps_{j} \cap \Eps_{j,good}}] \| = O \left( \rho_{j1} \exp(-\tau_j^2) \sqrt{\pi_j^*} \sqrt{\frac{d}{n} \log(dk^2T/\delta)}\right).
    \end{align*}

    \item $w_{i,1} X_i e_i \indic_{\Eps_{j} \cap \Eps_{j,1}^c}$, $j \neq 1$: The challenge here is how to bound $\E_{e \sim \mathcal{N} (0,1) } [|e|^p | |e| \ge \tau_j]$. We will use the standard lower bound for Gaussian tail bound: 
    $$P(e \ge \tau_j) \ge \frac{\tau_j}{\tau_j^2 + 1} \frac{1}{\sqrt{2\pi}} \exp(-\tau_j^2/2) \ge \frac{\exp(-\tau_j^2/2)}{3 \tau_j}.$$ 
    Now for the sub-exponential norm of $W = w_1 X e | \Eps_j \cap \Eps_{j,1}^c$ is given as
    \begin{align*}
        \|W\|_{\psi_1} &= \sup_{p \ge 1} p^{-1} \E_{\D_j}[|w_1 \vdot{X}{s} e|^p | \Eps_{j,1}^c]^{1/p} \\
        &\le \sup_{p \ge 1} p^{-1} \E_{\D_j} [|\vdot{X}{s}|^{p}]^{1/p} \E_{\D_j} [|e|^{p} | \Eps_{j,1}^c]^{1/p} \\
        &= \sup_{p \ge 1} p^{-1} \E_{\D_j} [|\vdot{X}{s}|^{p}]^{1/p} \left( \E_{\D_j} [|e|^{p} \indic_{\Eps_{j,1}^c} ] / P(\Eps_{j,1}^c) \right)^{1/p}.
    \end{align*}
    where in the inequality we used the independence of $X$ and $e$. $\E_{e \sim \mathcal{N}(0,1)} [|e|^{p} \indic_{e \ge \tau_j}]$ can be upper-bounded as follows:
    \begin{align*}
        \E_{e \sim \mathcal{N}(0,1)} [|e|^{p} \indic_{e \ge \tau_j}] &= \E_{e \sim \mathcal{N}(0,1)} [|e|^{p} \indic_{2\tau_j \ge |e| \ge \tau_j} ] + \E_{e \sim \mathcal{N}(0,1)} [|e|^{p} \indic_{|e| \ge 2\tau_j} ] \\
        &\le (2\tau_j)^{p} P(|e| \ge \tau_j) + \sqrt{\E[|e|^{2p}]} \sqrt{P(|e| \ge 2\tau_j)}.
    \end{align*}
    If we compare $\sqrt{P(|e| \ge 2\tau_j)}$ and $P(|e| \ge \tau_j)$, we can see that 
    \begin{align*}
        \sqrt{P(|e| \ge 2\tau_j)} / P(|e| \ge \tau_j) \le 3\tau_j \exp(-\tau_j^2/2).
    \end{align*}
    Now we can plug those values we found to proceed
    \begin{align*}
        \|W\|_{\psi_1} &\le \sup_{p \ge 1} p^{-1} \E_{\D_j} [|\vdot{X}{s}|^{p}]^{1/p} \left( (2\tau_j)^p + \sqrt{\E[|e|^{2p}]} 3\tau_j \exp(-\tau_j^2/2) \right)^{1/p} \\
        &\le c_0 \sup_{p \ge 1} p^{-1/2} \left( 2\tau_j + \sqrt{\E[|e|^{2p}]}^{1/p} (3\tau_j \exp(-\tau_j^2/2))^{1/p} \right) \\
        &\le C \tau_j,
    \end{align*}
    for some universal constant $c_0, C$. Now we have the sub-exponential norm of $W$, we can follow the procedure for $w_{1,i} X_i \vdot{X_i}{\beta_j^* - \beta_1^*} \indic_{\Eps_j \cap \Eps_{j,1}^c}$. Similarly to previously guaranteed in Remark \ref{remark:high_probability_on_n}, the deviation will be given as
    $$
        O\left( \tau_j \exp(-\tau_j^2 / (4c)) \sqrt{\frac{d}{n} \log^2(dk^2T/\delta)} \right).
    $$
    Again, we may require $c=4$ and $n > (k^2T/\delta)^{1/3}$ to get $\delta/(k^2T)$ probability bound.

    \item $w_{i,1} X_i e_i \indic_{\Eps_{j} \cap \Eps_{j,1} \cap (\Eps_{j,2} \cap \Eps_{j,3})^c)}$, $j \neq 1$: For this case, we set $W_i = w_{i,1} X_i e_i \indic_{\Eps_{j,1}} | \Eps_{j} \cap (\Eps_{j,2} \cap \Eps_{j,3})^c$ and find that 
    \begin{align*}
        \|W\|_{\psi_1} &= \sup_{p \ge 1} p^{-1} \E_\D[|w_1 \vdot{X}{s} e|^p \indic_{\Eps_{j,1}} | \Eps_{j} \cap (\Eps_{j,2} \cap \Eps_{j,3})^c]^{1/p} \\
        &= \sup_{p \ge 1} p^{-1} \E_{\D_j} [|w_1 \vdot{X}{s} e|^p \Eps_{j,1} | (\Eps_{j,2} \cap \Eps_{j,3})^c]^{1/p} \\
        &\le \sup_{p \ge 1} p^{-1} \E_{\D_j} [|\vdot{X}{s} e|^p| \Eps_{j,2}^c \cup \Eps_{j,3}^c]^{1/p} \\
        &\le \sup_{p \ge 1} p^{-1} \left (\sqrt{\E_{\D_j} [|\vdot{X}{s}|^{2p} | \Eps_{j,2}^c \cup \Eps_{j,3}^c]} \sqrt{\E[|e|^{2p}| \Eps_{j,2}^c \cup \Eps_{j,3}^c ]} \right)^{1/p} \\
        &\le \sup_{p \ge 1} p^{-1} \left (\sqrt{\E_{\D_j} [|\vdot{X}{s}|^{2p} | \Eps_{j,2}^c] + \E_{\D_j} [|\vdot{X}{s}|^{2p} | \Eps_{j,3}^c]} \sqrt{\E[|e|^{2p}]} \right)^{1/p} \\
        &\le C,
    \end{align*}
     for some constant $C > 0$. The probability is bounded as $P(\Eps_j \cap (\Eps_{j,2} \cap \Eps_{j,3})^c) \le \pi_j^* O\left(D_m/R_{j1}^* + \tau_j / R_{j1}^* \right)$, so we can bound the deviation in this case as 
     \begin{align*}
        \| \frac{1}{n} \sum_i w_{1,i} X_i e_i \indic_{\Eps_{j} \cap \Eps_{j,2}^c} - E[w_{1,i} X_i e_i \indic_{\Eps_{j} \cap \Eps_{j,2}^c}] \| = O \left( \sqrt{\frac{\pi_j^*}{k\rho_{\pi}} \vee \frac{\log(d k^2T/\delta)}{n}} \sqrt{\frac{d \log(dk^2T/\delta)}{n} }\right).
    \end{align*} 
     given our initialization and SNR condition.
     
    \item $w_{i,1} X_i e_i \indic_{\Eps_1}$ ($j = 1$): Finally, it is the easiest case since 
    \begin{align*}
        \|W\|_{\psi_1} &= \sup_{p \ge 1} p^{-1} \E_\D[|w_1 \vdot{X}{s} e|^p | \Eps_{1}]^{1/p} \\
        &= \sup_{p \ge 1} p^{-1} \E_{\D_1} [|w_1 \vdot{X}{s} e|^p]^{1/p} \\
        &\le \sup_{p \ge 1} p^{-1} \E_{\D_1} [|\vdot{X}{s} e|^p]^{1/p} \\
        &\le \sup_{p \ge 1} p^{-1} \left (\sqrt{\E_{\D_1} [|\vdot{X}{s}|^{2p}]} \sqrt{\E_{\D_1} [|e|^{2p}]} \right)^{1/p} \\
        &\le c_3,
    \end{align*}
    for some constant $c_3$. We can apply the same trick and get
     \begin{align*}
        \| \frac{1}{n} \sum_i w_{1,i} X_i e_i \indic_{\Eps_{1}} - E[w_{1,i} X_i e_i \indic_{\Eps_{1}}] \| = O \left( \sqrt{\pi_1^* \vee \frac{\log(d k^2T/\delta)}{n}} \sqrt{\frac{d \log(dk^2T/\delta)}{n} }\right).
    \end{align*} 
\end{enumerate}
    
Now we collect every scale of deviations from each item, and conclude that with probability $1-\delta/kT$ (by taking union bound over $O(k)$ items), 
we have
\begin{align}
    \label{eq:e_B_final_bound}
    e_B &= \frac{1}{n} \sum_i w_{1,i} X_i (Y_i - \vdot{X}{\beta_1^*}) - \E_\D [w_1 X (Y - \vdot{X}{\beta_1^*})] \nonumber \\
    &\le  \sqrt{\frac{d}{n} \log^2 (d k^2T /\delta)} \Bigg (\sum_{j \neq 1}^k  \sqrt{\pi_j^*} \rho_{j1} R_{j1}^* \exp(-\tau_j^2) + R_{j1}^* \exp(-\tau_j^2/16) + D_m \sqrt{\frac{\pi_j^* D_m}{R_{j1}^*} \vee \frac{1}{n}}  \nonumber \\
    &+ (1 \vee \rho_{j1}) \sqrt{\frac{1}{n} \vee \frac{\pi_j^* \tau_j}{R_{j1}^*}} + \sqrt{\pi_j^*} \rho_{j1} \exp(-\tau_j^2) + \tau_j \exp(-\tau_j^2/16) + \sqrt{\frac{\pi_j^*}{k \rho_\pi} \vee \frac{1}{n}} \Bigg) \nonumber \\
    &+ \sqrt{\frac{d \log(dk^2T/\delta)}{n}} \sqrt{\pi_1^*}.
\end{align}

As we set $\tau_j = c_\tau \sqrt{\log(k \rho_\pi R_{j1}^*)}$, SNR and initialization condition
\begin{align*}
    R_{j1}^* &= \Omega \left(k \rho_\pi \ \log(\rho_\pi k R_{j1}^*) \right) = \tilde{\Omega}(k), \\   
    D_m / R_{j1}^* &\le O(1 / (k\rho_\pi)), 
\end{align*}
and sample complexity 
$$n \gg \left(k/\pi_{min}^* (d / \epsilon^2) \log^2 (d k^2T/\delta) \right) \vee (k^2T/\delta)^{1/3},$$ every term inside the summation in \eqref{eq:e_B_final_bound} is now less than $O(\sqrt{\pi_1^* / k})$ or $O(D_m \sqrt{\pi_1^* / k})$. Thus, 
$$e_B \le O\left( \epsilon \sqrt{\frac{\pi_{min}}{k}} \left( k (1+D_m) \sqrt{\pi_1^*/k} \right) + \epsilon \pi_1^* \right).$$ 
We can conclude that $e_B \le \pi_1^* D_m \epsilon + \pi_1^* \epsilon$ with probability at least $1 - \delta/kT$.

\subsection{Concentration of $A$}
\label{appendix:concentrate_A}
As we only are interested in lower bound of the minimum eigenvalue, we only need to consider the concentration of $w_{1,i} X_i X_i^{\top} \indic_{\Eps_j}$ since $\frac{1}{n} \sum_i w_{1,i} X_i X_i^{\top} \succeq \frac{1}{n} \sum_i w_{1,i} X_i X_i^{\top} \indic_{\Eps_1}$. The concentration comes from standard concentration argument for random matrix with sub-exponential norm \cite{vershynin2010introduction}. For any fixed $s \in \mathbb{S}^{d-1}$, we have 
$$
    \| w_1 \vdot{X}{s}^2 \|_{\psi_1} \le 2 \|\vdot{X}{s} \|_{\psi_2}^2 \le K,
$$
with some universal constant $K$, since $w_1$ is bounded in [0,1]. Using this and $(1/2)$ covering-net argument over unit sphere, and the same argument we used with Proposition \ref{lemma:indic_prob_decompose}, we get
$$
    \big\| \frac{1}{n} \sum_i w_{1,i} X_i X_i^{\top} \indic_{\Eps_1} - \E_{\D} [w_1 XX^{\top} \indic_{\Eps_1}] \big\|_{op} \le O\left( \sqrt{\pi_1^*} \sqrt{\frac{d \log(k^2T/\delta)}{n}} \right),
$$
with high probability. As we see in the proof of Appendix \ref{Appendix:population_EM_A}, $\E_{\D} [w_1 XX^{\top} \indic_{\Eps_1}] \succeq (\pi_1^* / 2) I $ with good initialization and SNR. Thus, 
$$
    \frac{1}{n} \sum_i w_{1,i} X_i X_i^{\top} \succeq \frac{\pi_1^*}{2}I - \sqrt{\frac{\pi_1^* d \log (k^2T/\delta)}{n}} I \succeq \left( \frac{\pi_1^*}{2} - \epsilon \pi_1^* \right) I,
$$
given $n = \Omega(d \log(k^2T/\delta) / \pi_{min})$. Thus, we can get $\|A_n^{-1}\|_{op} \le 2/\pi_1^*$ with high probability.

\subsection{Concentration of Mixing Weights}
\label{appendix:concentrate_weights}

We can again use the per-sample decomposition strategy. The target we will bound is the error $\Big| \frac{1}{n} \sum_i w_{1,i} - \E_\D [w_1] \Big|$. As before, decompose $w_{1,i}$ as 
$$
    w_{1,i} = \left( \sum_{j>1}^k w_{1,i} \indic_{\Eps_j \cap \Eps_{j,good}} + w_{1,i} \indic_{\Eps_j \cap \Eps_{j,good}^c} \right) + w_{1,i} \indic_{\Eps_1}.
$$
It is the repetition of proofs for other quantities we have considered so far. 

\begin{enumerate}
     \item $w_{i,1} \indic_{\Eps_{j} \cap \Eps_{j,good}}$, $j \neq 1$: The difference is, now in all cases $W$ is a sub-Gaussian random variable. Note that $w_1$ is always less than 1. 
     \begin{align*}
        \| W \|_{\psi_1} &= \sup_{p \ge 1} p^{-1/2} \E_\D[|w_1|^p | \Eps_{j} \cap \Eps_{j,good} ]^{1/p} \\
        &= \sup_{p \ge 1} p^{-1/2} \E_{\D_j} [|w_1|^p | \Eps_{j,good}]^{1/p} \\
        &\le C \rho_{j1} \exp(-\tau_j^2),
    \end{align*}
     Following the same trick we used with Proposition \ref{lemma:indic_prob_decompose},with probability at least $1 - \delta/(k^2T)$, we get
     \begin{align*}
        | \frac{1}{n} \sum_i w_{1,i} \indic_{\Eps_{j} \cap \Eps_{j,good}} - E[w_{1,i} \indic_{\Eps_{j} \cap \Eps_{j,good}}] | = O \left( \rho_{j1} \exp(-\tau_j^2) \sqrt{\pi_j^*} \sqrt{\frac{1}{n} \log(k^2T/\delta)}\right).
    \end{align*}
     
     \item $w_{i,1} \indic_{\Eps_{j} \cap \Eps_{j,good}^c}$, $j \neq 1$: For this case, we set $W = w_{i,1} | \Eps_{j} \cap \Eps_{j,good}^c$ and find that 
     \begin{align*}
        \|W\|_{\psi_2} &= \sup_{p \ge 1} p^{-1/2} \E_\D[|w_1|^p \Eps_{j} \cap \Eps_{j,good}^c]^{1/p} \le 1.
    \end{align*}
     The probability is bounded as $P(\Eps_j \cap \Eps_{j,good}^c) \le \pi_j^* O\left(\exp(-\tau_j^2/2) + D_m/R_{j1}^* + \tau_j / R_{j1}^* \right) \le O(\pi_j^* / (k\rho_\pi))$, so we can bound the deviation in this case as 
     \begin{align*}
        | \frac{1}{n} \sum_i w_{1,i} \indic_{\Eps_{j} \cap \Eps_{j,2}^c} - E[w_{1,i} \indic_{\Eps_{j} \cap \Eps_{j,2}^c}] | = O \left( \sqrt{\frac{\pi_j^*}{k\rho_{\pi}} \vee \frac{\log(k^2T/\delta)}{n}} \sqrt{\frac{\log(k^2T/\delta)}{n} }\right).
    \end{align*} 
     given our initialization and SNR condition.
     
    \item $w_{i,1} \indic_{\Eps_1}$: Finally, it is the easiest case since 
    \begin{align*}
        \|W\|_{\psi_2} &= \sup_{p \ge 1} p^{-1/2} \E_\D[|w_1|^p | \Eps_{1}]^{1/p} \le 1,
    \end{align*}
    We can apply the same trick and get
     \begin{align*}
        | \frac{1}{n} \sum_i w_{1,i} \indic_{\Eps_{1}} - E[w_{1,i} \indic_{\Eps_{1}}] | = O \left( \sqrt{\pi_1^* \vee \frac{\log(k^2T/\delta)}{n}} \sqrt{\frac{\log(k^2T/\delta)}{n} }\right).
    \end{align*} 
\end{enumerate}

Now combining this all, given $n = \Omega(k\epsilon^{-2}/\pi_{min})$ we have
\begin{align*}
    \left| \frac{1}{n} \sum_i w_{1,i} - \E_{\D} [w_1] \right| &\le \sqrt{\frac{1}{n} \log(k^2T/\delta)} \left( \sum_{j>1}^k \rho_{j1} \exp(-\tau_j^2) \sqrt{\pi_j^*} + \sqrt{\frac{\pi_1^*}{k}} \right) + \sqrt{ \frac{\pi_1^* \log(k^2T/\delta)}{n}} \\
    &\le \epsilon \sqrt{\frac{\pi_1^*}{k}} \left( \sum_{j>1}^k \frac{\sqrt{\rho_{j1}} \sqrt{\pi_1^*} }{k \rho_{\pi}} + \sqrt{\frac{\pi_1^*}{k}} \right) + \epsilon \pi_1^* \le O(\pi_1^* \epsilon).
\end{align*}
This implies the concentration of mixing weights in relative scale.

\section{Proof of Auxiliary Lemmas}
\label{appendix:technical_lemmas}

\anglecomp*
\begin{proof}
    Equation \eqref{eq:aux_angle_1} is a consequence of Lemma 6 in \cite{yi2016solving} and elementary rule of union bounds.

    For \eqref{eq:aux_angle_2}, we first look at $p^{th}$ moment conditioned on only one event. Recall that in \cite{yi2016solving}, only the case for $p=2$ is proven. Without loss of generality, due to the rotational invariance of standard Gaussian distribution, we can assume span($u, v_1$) = span($\bm{e}_1, \bm{e}_2$). 
    
    Change first two coordiantes of $X$, $x_1, x_2$ to combination of $r$ Rayleigh distribution and $\theta$ uniformly distributed over $[0, 2\pi)$. Then define $Y = \vdot{s_{3:d}}{X_{3:d}}$ where $s_{3:d}, X_{3:d}$ be partial vectors of $s$ and $X$ from third coordinate. Then $Y \sim \mathcal{N}(0,\|s_{3:d}\|^2)$, and $r, \theta, Y$ are all independent.
    
    Now note that the event $\Eps_1 = |\vdot{X}{u_1}| \ge |\vdot{X}{v}|$ only depends on $\theta$. Then,
    \begin{align*}
        \E[\vdot{X}{s}^p | \Eps_1^c] &= \E[|s_1 r \cos \theta + s_2 r \sin \theta + Y|^p | \Eps_1^c] \\
        &= \frac{\E[|s_1 r \cos \theta + s_2 r \sin \theta + Y|^p \indic_{\Eps_1^c}]}{P(\Eps_1^c)} \\
        &= \frac{\E_{\theta}[\E_{r, Y}[|r s_1 \cos \theta + r s_2 \sin \theta + Y|^p|\theta] \ \indic_{\theta \in \Eps_1^c}]}{P(\Eps_1^c)} \\
        &= \frac{\E_{\theta}[(\E_{r, Y}[|r s_1 \cos \theta + r s_2 \sin \theta + Y|^p|\theta]^{1/p})^p \ \indic_{\theta \in \Eps_1^c}]}{P(\Eps_1^c)} \\
        &\underset{(i)}{\le} \frac{\E_{\theta}[(\E_{r}[|r s_1 \cos \theta + r s_2 \sin \theta|^p | \theta]^{1/p} + \E_Y[|Y|^p|\theta]^{1/p})^p \ \indic_{\theta \in \Eps_1^c}]}{P(\Eps_1^c)} \\
        &\underset{(ii)}{\le} \frac{\E_{\theta}[(\E_r[r^p |s_1 \cos \theta + s_2 \sin \theta|^p | \theta]^{1/p} + \E_Y[|Y|^p]^{1/p})^p \ \indic_{\theta \in \Eps_1^c}]}{P(\Eps_1^c)} \\
        &{\le} \frac{\E_\theta[\E_r[r^p]^{1/p} \|s_{1:2}\| + \E_Y[|Y|^p]^{1/p})^p \ \indic_{\theta \in \Eps_1^c}]}{P(\Eps_1^c)} \\
        &\underset{(iii)}{=} \frac{(\E_r[r^p]^{1/p} \|s_{1:2}\| + \E_Y[|Y|^p]^{1/p})^p \ \E_\theta[\indic_{\theta \in \Eps_1^c}]}{P(\Eps_1^c)} \\
        &= (\E[r^p]^{1/p} \|s_{1:2}\| + \E[|Y|^p]^{1/p})^p,
    \end{align*}
    where (i) we used Minkowski's inequality, (ii) we used independence of $\theta$ and $Y$, and (iii) used independence of all terms from $\theta$.
    
    Then, since $r \sim \text{Rayleigh(1)}$ and $Y \sim \mathcal{N}(0,\|s_{3:d}\|^2)$, we have an exact value for each $p^{th}$ moments from well-known distribution properties. That is, 
    \begin{align*}
        \E[\vdot{X}{s}^p | \Eps^c] &\le \left( \|s_{1:2}\| \sqrt{2} \Gamma(1+p/2)^{1/p} + \|s_{3:d}\| \sqrt{2} \left(\Gamma ((p+1)/2) / \sqrt{\pi} \right)^{1/p} \right)^p.
    \end{align*}
    Now since $\Gamma(1+ p/2) \ge 2 \Gamma((p+1) / 2)/\sqrt{\pi}$ for $p \ge 1$, and $$\|s_{1:2}\| + \|s_{3:d}\| \le \sqrt{\|s_{1:2}\|^2 + \|s_{3:d}\|^2} \sqrt{2} = \sqrt{2}$$ since $s$ is an unit vector, we conclude that 
    $$
        \E[\vdot{X}{s}^p | \Eps_1^c] \le 2^p \Gamma(1+p/2).
    $$
    
    Now we move on to conditioning on $\Eps^c$. It comes from elementary property of union of the events,
    \begin{align*}
        \E[|\vdot{X}{s}|^p | \Eps^c] &= \frac{\E[|\vdot{X}{s}|^p \indic_{\Eps^c}]}{P(\Eps^c)} \le \frac{\E[|\vdot{X}{s}|^p \sum_i \indic_{\Eps_i^c}]}{P(\Eps^c)} \\
        &= \sum_i \frac{\E[|\vdot{X}{s}|^p \indic_{\Eps_i^c}]}{P(\Eps^c)}
        \le \sum_i \frac{\E[|\vdot{X}{s}|^p \indic_{\Eps_i^c}]}{P(\Eps_i^c)} \\
        &\le k 2^{p} \Gamma(1+p/2),
    \end{align*}
    where we used $P(\Eps^c) \ge P(\Eps_i^c)$, and $\indic_{\Eps^c} \le \sum_i \indic_{\Eps_i^c}$ since $\Eps^c = \cup_i \Eps_i^c$. The claim follows.
\end{proof}

\thresholdcomp*
\begin{proof}
    Equation \eqref{eq:aux_threshold_1} is a direct consequence of lemma 9(v) in \cite{balakrishnan2017statistical} and union bound.

    We start of \eqref{eq:aux_threshold_2} with the same strategy in proof of \ref{lemma:angle_comparison}. Let us consider only one comparison first. Let $\Eps_1 = \{ |\vdot{X}{u_1} \ge \alpha_1 \}$. Without loss of generality (by rotational invariance of standard Gaussian), let $u_1 = \bm{e}_1$ and $Y = \vdot{x_{2:d}}{s_{2:d}}$. 
    \begin{align*}
        \E[|\vdot{X}{s}|^p | \Eps_1^c] &= \E[|s_1 x_1 + Y|^p | (|x_1| \le \alpha_1)] \\
        &= \frac{\E[|s_1 x_1 + Y|^p \indic_{x_1 \le \alpha_1}]}{P(|x_1| \le \alpha_1)} \\
        &\le \frac{\E[\E[|s_1 x_1 + Y|^p | x_1] \indic_{x_1 \le \alpha_1}]}{P(x_1 \le \alpha_1)} \\
        &\le \frac{\E[(\E[|s_1 x_1|^p | x_1]^{1/p} + \E[|Y|^p | x_1]^{1/p})^p | x_1] \indic_{x_1 \le \alpha_1}]}{P(x_1 \le \alpha_1)} \\
        &\le \frac{\E[ (|s_1 x_1| + \E[|Y|^p]^{1/p})^p \indic_{x_1 \le \alpha_1}]}{P(x_1 \le \alpha_1)} \\
        &\le \frac{\E[ (|s_1 \alpha_1| + \sqrt{2} \| s_{2:d} \| (\Gamma((1+p)/2)/\sqrt{\pi})^{1/p})^p \indic_{x_1 \le \alpha_1}]}{P(x_1 \le \alpha_1)} \\
        &= \left(|s_1 \alpha_1| + \sqrt{2} \| s_{2:d} \| (\Gamma((1+p)/2)/\sqrt{\pi})^{1/p} \right)^p \\
        &\le 2^p \Gamma((1+p)/2) / \sqrt{\pi}.
    \end{align*}
    
    The rest of the proof follows by decomposing union events into separate events as before.
    \begin{align*}
        \E[|\vdot{X}{s}|^p | \Eps^c] &= \frac{\E[|\vdot{X}{s}|^p \indic_{\Eps^c}]}{P(\Eps^c)} \le \frac{\E[|\vdot{X}{s}|^p \sum_i \indic_{\Eps_i^c}]}{P(\Eps^c)} \\
        &= \sum_i \frac{\E[|\vdot{X}{s}|^p \indic_{\Eps_i^c}]}{P(\Eps^c)}
        \le \sum_i \frac{\E[|\vdot{X}{s}|^p \indic_{\Eps_i^c}]}{P(\Eps_i^c)} \\
        &\le k 2^{p} \Gamma((1+p)/2) / \sqrt{\pi}.
    \end{align*}
\end{proof}

\probdecomp*
\begin{proof}
    We are interested in bounding the following probability
    \begin{align*}
        P\left(\| \sum_i \left(X_i \indic_{A} - \E[X_i \indic_{A}]\right)\|_2 \ge nt\right).
    \end{align*}
    We will upper bound this probability by spliting it with conditioning on every possible set of Bernoulli variables $Z_i$, then arrange them.
    \begin{align*}
        P\Bigg(\Big\| &\sum_i \Big(X_i \indic_{A} - \E[X_i \indic_{A}]\Big) \Big\| \ge nt\Bigg) = \sum_{\{Z_i\}_1^n \in \{0,1\}^n} P \left(\| \sum_i (X_i Z_i - \E[X_i \indic_{A}])\| \ge nt \Big| \{ Z_i \}_1^n \right) P(\{ Z_i \}_1^n).
    \end{align*}
    Note that $X_i Z_i = 0$ when $Z_i = 0$, and $X_i Z_i = X_i | A = Y_i$ when $Z_i = 1$. Now we divide the cases into when $\sum_i Z_i \le n_e$ and $\sum_i Z_i > n_e$.
    \begin{align*}
        & \sum_{\{Z_i\}_1^n \in \{0,1\}^n} P \left( \|(\sum_{i: Z_i = 1} X_i) - n \E[X_i | A] P(A)\| \ge nt \Big| \{ Z_i \}_1^n \right) P(\{ Z_i \}_1^n) \\
        &\le \sum_{\{Z_i\}_1^n \in \{0,1\}^n, \sum_i Z_i \le n_e} P\left( \|(\sum_{i: Z_i = 1} X_i) - n\E[X | A] P(A)\| \ge nt \Big| \{ Z_i \}_1^n \right) P(\{ Z_i \}_1^n) + P\left(\sum_i Z_i \ge n_e + 1\right).
    \end{align*}
    We can decouple the first term above into two terms as the following:
    \begin{align*}
        P &\left( \|(\sum_{i: Z_i = 1} X_i) - n\E[X | A] P(A)\| \ge nt \Big| \{ Z_i \}_1^n \right) \\
        &= P \left( \|\sum_{i: Z_i = 1} (X_i - \E[X | A]) + \E[X | A](\sum_i Z_i - nP(A))\| \ge nt \Big| \{ Z_i \}_1^n \right) \\
        &\le P \left( \|\sum_{i: Z_i = 1} (X_i - \E[X | A])\| \ge nt_1 
        \Big| \{ Z_i \}_1^n \right) + P\left( \|\E[X | A](\sum_i Z_i - nP(A))\| \ge nt_2 \Big| \{ Z_i \}_1^n \right).
    \end{align*}
    where $t_1 + t_2 = t$. Then we observe that conditioned on $Z_i = 1$, $X_i$ is actually $Y_i$, and we can discard all $X_i$ for $i$ such that $Z_i = 0$. Thus, the first expression is simplified to
    \begin{align}
        \label{eq:X_bar_Z_equal_Y}
        P \left( \|\sum_{i: Z_i = 1} (X_i - \E[X | A])\| \ge nt_1 
        \Big| \{ Z_i \}_1^n, \sum_i Z_i = m \right) &= P\left( \left\| \sum_{j=1}^{m} (Y_j - \E[Y]) \right\| \ge n t_1 \right),
    \end{align}
    Here, $j$ is a new index variable, and now the condition is only on the sum of $Z_i$, which is $m$. Now we are ready to wrap up the result:
    \begin{align*}
        P\Bigg(\Big\| &\sum_i \Big(X_i \indic_{A} - \E[X_i \indic_{A}]\Big) \Big\| \ge nt\Bigg) \\
        &\le \sum_{\{Z_i\}_1^n \in \{0,1\}^n, \sum_i Z_i \le n_e} P\left( \left\| \sum_{j=1}^{m} (Y_j - \E[Y]) \right\| \ge n t_1 \right) P(\{Z_i\}_1^n, \sum_i Z_i = m) \\ 
        &+ \sum_{\{Z_i\}_1^n \in \{0,1\}^n, \sum_i Z_i \le n_e} P\left( \|\E[Y]\| \Big|\sum_i Z_i - nP(A) \Big| \ge n t_2 \Big| \{Z_i\}_1^n \right) P(\{Z_i\}_1^n) \\
        &+ P\left(\sum_i Z_i \ge n_e + 1 \right) \\
        &\le \max_{m \le n_e} P\left( \left\| \sum_{j=1}^{m} (Y_j - \E[Y]) \right\| \ge n t_1 \right) \\
        &+ P\left( \|\E[Y]\| \Big|\sum_i Z_i - nP(A) \Big| \ge n t_2\right) + P\left(\sum_i Z_i \ge n_e + 1 \right),
    \end{align*}
    where the last inequality we used the fact $\sum_{\{Z_i\}_1^n \in \{0,1\}^n} P(\{Z_i\}_1^n) = 1$, and \eqref{eq:X_bar_Z_equal_Y} is only conditioned on the sum of $Z_i$ being less than $n_e$. We divide by $n$ in conditions inside the first two probabilities, and we get the theorem.
\end{proof}

\section{Defered Proof: Bounding $B$ for population EM when $D_m \le 1$}
\label{appendix:population_EM_B_II}

\paragraph{Case II. $\max_j \| \Delta_j \| \le 1$:} 
We use mean-value theorem to reformulate $\Delta_w$. We additionally define a symbol $\delta_j := \pi_j - \pi_j^*$. Denote $\beta_j^u = \beta_j^* + u \Delta_j$ and $\pi_j^u = \pi_j^* + u \delta_j$ for $u \in [0,1]$, and let $w_1^u$ be the weight in E-step constructed with $\beta_j^u$ and $\pi_j^u$. Then, by mean-value theorem, for some $u \in [0,1]$, $B = \| \E_\D [\Delta_w^u \vdot{X}{s} (Y - \vdot{X}{\beta_1^*})] \|$ where
\begin{align*}
    \Delta_w^u &= \underbrace{-w_1^u (1-w_1^u) (\vdot{X}{\beta_1^u} - Y) \vdot{X}{\Delta_1} + \sum_{l \neq 1} w_1^u w_l^u (\vdot{X}{\beta_l^u} - Y) \vdot{X}{\Delta_l}}_{\Delta_{w,1}} \\ &\qquad \underbrace{-w_1^u (1-w_1^u) \delta_1 / \pi_1^u + \sum_{l \neq 1} w_1^u w_l^u \delta_l / \pi_l^u}_{\Delta_{w,2}},
\end{align*}
for some $u \in [0,1]$. Note that $\delta_j / \pi_j^u \le 2\delta_j / \pi_j^* \le 1$ guaranteed by initialization condition and the result for $D_m \ge 1$. Let us now redefine $D_m = \max(\max_j \|\Delta_j\|, \max_j \delta_j / \pi_j^*)$. Then for each $j$, we can decompose the target term as
\begin{align*}
    |\E_{\D_j} [\Delta_w^u \vdot{X}{s} (Y - \vdot{X}{\beta_1^*})]| \le \underbrace{|\E_{\D_j} [\Delta_{w,1} \vdot{X}{s} (Y - \vdot{X}{\beta_1^*})]|}_{E_1} + \underbrace{|\E_{\D_j} [\Delta_{w,2} \vdot{X}{s} (Y - \vdot{X}{\beta_1^*})]|}_{E_2}.
\end{align*}
We will bound $E_1$ and $E_2$ separately as we proceed.

\paragraph{$j \neq 1$:} 
\paragraph{Bounding E1.} We first consider bounding the first term.
\begin{align*}
    E_1 &= |\E_{\D_j} [\Delta_{w,1} \vdot{X}{s} (Y - \vdot{X}{\beta_1^*})]| \\
    &\le \underbrace{|\E_{\D_j} [\Delta_{w,1} \vdot{X}{s} \vdot{X}{\beta_j^* - \beta_1^*}]|}_{b_1} + \underbrace{|\E_{\D_j} [\Delta_{w,1} \vdot{X}{s} e]|}_{b_2}, 
\end{align*}
As before, we will first bound $b_1$. It is a bit complicated as it involves many algebraic terms, but the idea is the same.
\begin{align*}
    b_1 &\le \underbrace{\left| \E_{\D_j} [w_1^u (\vdot{X}{\beta_j^* - \beta_1^u} + e) \vdot{X}{\Delta_1} \vdot{X}{s} \vdot{X}{\beta_j^* - \beta_1^*}] \right|}_{d_1} \\
    &+ \underbrace{\left| \E_{\D_j} \left[\sum_{l=1}^k w_1^u w_l^u (\vdot{X}{\beta_j^* - \beta_l^u} + e) \vdot{X}{\Delta_l} \vdot{X}{s} \vdot{X}{\beta_j^* - \beta_1^*} \right] \right|}_{d_2}.
\end{align*}

We bound $d_2$ first. Consider the following good events:
\begin{align*}
    & \Eps_1 = \{|\vdot{X}{\Delta_l}| \le D_m \tau_j, \forall l\} \cap \{|e| \le \tau_j \}, \ 
    \Eps_2 = \{|\vdot{X}{\beta_j^* - \beta_1^u}| \ge 4 \tau_j \}.
\end{align*}
We will set $\tau_j = c_\tau \left(\sqrt{\log(R_{j1}^* k \rho_\pi)} \right)$ for some large constant $c_\tau > 0$.

Under event $\Eps_1$, when $l \neq j$, we claim $|w_l^u (\vdot{X}{\beta_j^* - \beta_l^u} + e)| \le \rho_{jl} |\exp(-6 \tau_j^2) 4\tau_j| + w_l^u 4\tau_j$. Let us denote $r := (\vdot{X}{\beta_j^* - \beta_l^u} + e)$. Then
\begin{align*}
    w_l^u &\le \rho_{jl} \exp\left( \frac{-(\vdot{X}{\beta_j^* - \beta_l^u} + e)^2 + (\vdot{X}{\beta_j^* - \beta_j^u} + e)^2}{2} \right) \\
    &= \rho_{jl} \exp\left( (\vdot{X}{\beta_j^* - \beta_j^u} + e)^2/2 \right) \exp\left( -(\vdot{X}{\beta_j^* - \beta_l^u} + e)^2/2 \right) \\
    &= \rho_{jl} \exp\left( 2\tau_j^2 \right) \exp\left( -r^2/2 \right).
\end{align*}
Thus $|w_l^u r| \le \rho_{jl} \exp(2\tau_j^2) r\exp(-r^2/2)$. The function $f(r) = r\exp(-r^2/2)$ is maximized when $r=1$, and decreasing afterward. Therefore, we can conclude that whenever $r > 4\tau_j$, 
$$
    |w_l^u r| \le \rho_{jl} \exp\left( 2\tau_j^2 \right) \sup_{r \ge 4\tau_j} r \exp\left( -r^2/2 \right) \le \rho_{jl} \exp(2\tau_j^2) 4\tau_j \exp(-8\tau_j^2) \le \rho_{jl} 4\tau_j \exp(-6\tau_j^2).
$$

When $4\tau_j > r$, we have $|w_l^u r| \le w_l^u 4\tau_j$. Thus, we have $|w_l^u r| \le  4 \rho_{jl} \tau_j \exp(-6\tau_j^2) + |w_l^u 4\tau_j|$. 

For $l = j$, under event $\Eps_1$, we know $|\vdot{X}{\Delta_j} + e| \le 2 \tau_j$. Thus, it is also true for $j = l$ that $|w_l^u r| \le (4\tau_j \exp(-6\tau_j^2) \vee |w_l^u| 4 \tau_j)$.

\allowdisplaybreaks
Now we plugging these relations into $d_2$, we get
\begin{align*}
    \E_{\D_j} &\left[\sum_{l} w_1^u w_l^u (\vdot{X}{\beta_j^* - \beta_l^u} + e) \vdot{X}{\Delta_l} \vdot{X}{s} \vdot{X}{\beta_j^* - \beta_1^*} \right] \\ 
    &\le \rho_\pi \E_{\D_j} \left[\sum_{l} \left| w_1^u \exp(-6\tau_j^2) 4\tau_j \vdot{X}{\Delta_l} \vdot{X}{s} \vdot{X}{\beta_j^* - \beta_1^*} \right| \indic_{\Eps_1} \right] \\
    &+ \E_{\D_j} \left[\sum_{l} \left|w_1^u w_l^u 4\tau_j \vdot{X}{\Delta_l} \vdot{X}{s} \vdot{X}{\beta_j^* - \beta_1^*} \right| \indic_{\Eps_1} \right]  \\
    &+ \E_{\D_j} \left[ \left| \sum_{l} w_1^u w_l^u (\vdot{X}{\beta_j^* - \beta_l^u} + e) \vdot{X}{\Delta_l} \vdot{X}{s} \vdot{X}{\beta_j^* - \beta_1^*} \right| \indic_{\Eps_1^c} \right] \\
    &\le 4 \rho_\pi D_m\tau_j^2 \exp(-6\tau_j^2) \underbrace{\E_{\D_j} \left[\sum_{l} \left| w_1^u \vdot{X}{s} \vdot{X}{\beta_j^* - \beta_1^*} \right| \indic_{\Eps_1} \right] }_{(i)} \\
    &+ 4 D_m \tau_j^2 \underbrace{\E_{\D_j} \left[ \left| (\sum_l w_l^u) w_1^u \vdot{X}{s} \vdot{X}{\beta_j^* - \beta_1^*} \right| \indic_{\Eps_1} \right]}_{(ii)} \\
    &+ \underbrace{\E_{\D_j} \left[\left| \sum_{l} w_1^u w_l^u (\vdot{X}{\beta_j^* - \beta_l^u} + e) \vdot{X}{\Delta_l} \vdot{X}{s} \vdot{X}{\beta_j^* - \beta_1^*} \indic_{\Eps_1^c} \right| \right]}_{(iii)}.
\end{align*}

For (i), 
\begin{align*}
    \E_{\D_j} &\left[\sum_{l} \left| w_1^u \vdot{X}{s} \vdot{X}{\beta_j^* - \beta_1^*} \right| \indic_{\Eps_1} \right] \\
    &\le \sum_{l} \sqrt{\E_{\D_j} [\vdot{X}{\beta_j^* - \beta_1^*}^2]} \sqrt{\E_{\D_j} [\vdot{X}{s}^2]} = \sum_l R_{j1}^* = k R_{j1}^*.
\end{align*}

For (ii), 
\begin{align*}
    \E_{\D_j} \left[ \left| w_1^u \vdot{X}{s} \vdot{X}{\beta_j^* - \beta_1^*} \right| \indic_{\Eps_1} \right] &= \E_{\D_j} \left[ \left| w_1^u \vdot{X}{s} \vdot{X}{\beta_j^* - \beta_1^*} \right| \indic_{\Eps_1 \cap \Eps_2} \right] \\
    &+ \E_{\D_j} \left[ \left| w_1^u \vdot{X}{s} \vdot{X}{\beta_j^* - \beta_1^*} \right| \indic_{\Eps_1 \cap \Eps_2^c} \right].
\end{align*}
Under event $\Eps_1 \cap \Eps_2$, it is easy to see that 
\begin{align*}
    |\vdot{X}{\beta_j^* - \beta_j^u} + e| \le 2\tau_j, \ |\vdot{X}{\beta_j^* - \beta_1^u} + e| \ge 3\tau_j, \ w_1^u \le \rho_{j1} \exp(-2\tau_j^2),
\end{align*}
thus 
$$
\E_{\D_j} \left[ \left| w_1^u \vdot{X}{s} \vdot{X}{\beta_j^* - \beta_1^*} \right| \indic_{\Eps_1 \cap \Eps_2} \right] \le \rho_{j1} \exp(-2 \tau_j^2) R_{j1}^*.
$$
For the second term:
\begin{align*}
    \E_{\D_j} \left[ \left| w_1^u \vdot{X}{s} \vdot{X}{\beta_j^* - \beta_1^*} \right| \indic_{\Eps_1 \cap \Eps_2^c} \right] &\le \E[ |\vdot{X}{s}| |\vdot{X}{\beta_j^* - \beta_1^u} + u\vdot{X}{\Delta_1}| \indic_{\Eps_1 \cap \Eps_2^c} ] \\
    &\le \E[|\vdot{X}{s}| (5 \tau_j) \indic_{\Eps_1 \cap \Eps_2^c} ]  \\
    &\le 5\tau_j \E[ |\vdot{X}{s}| | \Eps_2^c] P(\Eps_2^c) \\
    &\le c_1 \tau_j^2 / R_{j1}^*.
\end{align*}

For (iii), note that $P(\Eps_1^c) \le 2k\exp(-\tau_j^2/2)$. Then,
\begin{align}
    \label{eq:intermidate_iii}
    \text{(iii)} &\le \E_{\D_j} \left[\left| \sum_{l} w_1^u w_l^u (\vdot{X}{\beta_j^* - \beta_l^u} + e) \vdot{X}{\Delta_l} \vdot{X}{s} \vdot{X}{\beta_j^* - \beta_1^*} \indic_{\Eps_1^c} \right| \right] \nonumber \\ 
    &\le \sum_{l} \sqrt{\E_{\D_j}[{w_l^u}^2 \vdot{X}{\beta_j^* - \beta_l^u} + e)^2 \vdot{X}{\Delta_l}^2} \sqrt{\E_{\D_j} \left[\vdot{X}{s}^2 \vdot{X}{\beta_j^* - \beta_1^*}^2 \indic_{\Eps_1^c} \right]} \nonumber \\
    &\le \sum_{l} \sqrt{\E_{\D_j}[(w_l^u)^2 (\vdot{X}{\beta_j^* - \beta_l^u} + e)^2 \vdot{X}{\Delta_l}^2 } \sqrt[8]{\E_{\D_j} [\vdot{X}{s}^8]} \sqrt[8]{\E_{\D_j} \left[\vdot{X}{\beta_j^* - \beta_1^*}^8 \right]} \sqrt[4]{P(\Eps_1^c)} \nonumber \\
    &\le c R_{j1}^* \sqrt[4]{k} \exp(-\tau_j^2/8) \left(\sum_{l} \sqrt{\E_{\D_j}[(w_l^u)^2 (\vdot{X}{\beta_j^* - \beta_l^u} + e)^2 \vdot{X}{\Delta_l}^2 } \right).
\end{align}

In order to bound \eqref{eq:intermidate_iii}, we need the following equation which we defer to prove in \ref{Appendix:population_supp}:
\begin{lemma}
    \label{lem:supplementary_Dm_less1}
    If $D_m \le 1$, for $j \neq l$,
    \begin{equation}
        \label{eq:inner_Eps_l}
        \E_{\D_j} \left[(w_l^u)^2 \vdot{X}{(\beta_j^* - \beta_l^u + e)}^2\vdot{X}{\Delta_l}^2\right] \le O\left((\rho_{jl} R_{jl}^*)^2 \exp(-\tau_l^2/2) + \tau_l^3 / R_{jl}^* \right) \| \Delta_l \|^2,
    \end{equation}
    which is less than $O(\| \Delta_l \|^2)$ with $\tau_l = O(\sqrt{ \log(R_{jl}^* \rho_\pi)})$.
    
    If $j = l$, we have
    \begin{equation}
        \label{eq:inner_Eps_l2}
        \E_{\D_j} \left[(w_j^u)^2 \vdot{X}{(\beta_j^* - \beta_j^u + e)}^2\vdot{X}{\Delta_j}^2\right] \le O\left(\tau_j^2 + (\|\Delta_j\|^2 + 1) \sqrt{k} \exp(-\tau_j^2/4) \right) \| \Delta_j \|^2,
    \end{equation}
    which is less than $O(\| \Delta_j \|^2 \log k)$ with $\tau_j = O(\sqrt{\log k})$.
\end{lemma}

Then, we can bound \eqref{eq:intermidate_iii} by
\begin{align*}
    \eqref{eq:intermidate_iii} &\le O\left( R_{j1}^* \sqrt[4]{k} \exp(-\tau_j^2/8) (\sum_{l\neq j} D_m + \sqrt{\log k} D_m) \right) \\
    &\le O \left(R_{j1}^* k^{5/4} \exp(-\tau_j^2/8) D_m \right).
\end{align*}

Combining all results, we have 
\begin{align*}
    d_2 \le O\left( (\rho_\pi k \tau_j^2 + k^{5/4}) \exp(-\tau_j^2/8) R_{j1}^* + \tau_j^4 / R_{j1}^* \right) D_m.
\end{align*}
Then, we set $\tau_j = \Theta \left(\sqrt{\log(R_{j1}^* k \rho_\pi)} \right)$ to get $d_2 \le c_d D_m / (k \rho_\pi)$ along with $R_{j1}^* \ge R_{min} \ge \tilde{\Omega}(k\rho_\pi)$.

Now for $d_1$, we follow the exactly same path, while the only difference is that it does not involve summation over all components. 
\begin{align*}
    d_1 &= \E_{\D_j} \left[ w_1^u(\vdot{X}{\beta_j^* - \beta_1^u} + e) \vdot{X}{\Delta_1} \vdot{X}{s} \vdot{X}{\beta_j^* - \beta_1^*} \right] \\
    &\le \rho_\pi \exp(-6\tau_j^2) 4\tau_j \left[ |\vdot{X}{\Delta_1} \vdot{X}{s} \vdot{X}{\beta_j^* - \beta_1^*}| \indic_{\Eps_1} \right] \\
    &+ 4\tau_j \E_{\D_j} \left[ |w_1^u \vdot{X}{\Delta_1} \vdot{X}{s} \vdot{X}{\beta_j^* - \beta_1^*}| \indic_{\Eps_1} \right] \\
    &+ \E \left[ |w_1^u (\vdot{X}{\beta_j^* - \beta_1^u} + e) \vdot{X}{\Delta_1} \vdot{X}{s} \vdot{X}{\beta_j^* - \beta_1^*}| \indic_{\Eps_1^c} \right] \\
    &\le O \left((\rho_\pi \tau_j^2 + \sqrt[4]{k}) \exp(-\tau_j^2/8) R_{j1}^*  + \tau_j^4 / R_{j1}^* \right) D_m,
\end{align*}
where we can set $\tau_j$ the same, and we get $d_1 \le c_{d'} D_m / (k\rho_\pi)$. Therefore we complete the proof for $b_1 \le c_b D_m / (k \rho_\pi)$.

The bound for $b_2$ is replicate of the proof for $b_1$ except that, at the end of inequality we get $\sqrt{\E[\vdot{X}{s}^2 e^2]}$ instead of $\sqrt{\E[\vdot{X}{s}^2 \vdot{X}{\beta_j^* - \beta_1^*}^2]}$. Specifically, we start from
\begin{align*}
    b_2 &\le \underbrace{|\E_{\D_j} [w_1^u (\vdot{X}{\beta_j^* - \beta_1^u} + e) \vdot{X}{\Delta_1} \vdot{X}{s} e]|}_{d_1} \\
    &+ \underbrace{|\E_{\D_j} [\sum_{l=1}^k w_1^u w_l^u (\vdot{X}{\beta_j^* - \beta_l^u} + e) \vdot{X}{\Delta_l} \vdot{X}{s} e]|}_{d_2}.
\end{align*}
For $d_2$, applying the same argument, we get
\begin{align*}
    \E_{\D_j} &\left[\sum_{l} w_1^u w_l^u (\vdot{X}{\beta_j^* - \beta_l^u} + e) \vdot{X}{\Delta_l} \vdot{X}{s} e \right] \\ 
    &\le 4 \rho_\pi D_m\tau_j^2 \exp(-6\tau_j^2) \underbrace{\E_{\D_j} \left[\sum_{l} \left| w_1^u \vdot{X}{s} e \right| \indic_{\Eps_1} \right] }_{(i)} + 4 D_m \tau_j^2 \underbrace{\E_{\D_j} \left[ \left| (\sum_l w_l^u) w_1^u \vdot{X}{s} e \right| \indic_{\Eps_1} \right]}_{(ii)} \\
    &+ \underbrace{\E_{\D_j} \left[\left| \sum_{l} w_1^u w_l^u (\vdot{X}{\beta_j^* - \beta_l^u} + e) \vdot{X}{\Delta_l} \vdot{X}{s} e \indic_{\Eps_1^c} \right| \right]}_{(iii)}.    
\end{align*}

Then, we can go through exactly same path to bound each (i), (ii), (iii). Finally, set $\tau_j = \Theta \left( \sqrt{\log(R_{j1}^* k \rho_\pi)} \right)$ as before and we get the bound $E_1 \le c_b D_m / (k \rho_\pi)$ for $j \neq 1$.

\paragraph{Bounding $E_2$, the term from mismatch in mixing weights.} Recall that 
\begin{align*}
    E_2 &= |\E_{\D_j} [\Delta_{w,2} \vdot{X}{s} (Y - \vdot{X}{\beta_1^*})]|, \\
    \Delta_{w,2} &= -w_1^u(1-w_1^u) \delta_1 / \pi_1^u + \sum_{l\neq1} w_1^u w_l^u \delta_l / \pi_l^u \\
    &\le \left| w_1^u(1-w_1^u) + \sum_{l\neq1} w_1^u w_l^u \right| 2 D_m = 2 w_1^u D_m.
\end{align*}
Hence, $E_2 \le 2 D_m \E_{\D_j} [w_1^u |\vdot{X}{s} (Y - \vdot{X}{\beta_1^*})|]$. We have already seen similar equation when we handle $D_m \ge 1$. Only difference is that $\Delta_w$ is now changed to $w_1^u$, but we can observe that we can reuse the exactly same procedure. (Remember the only property we needed for $\Delta_w$ was that it to be less than $\exp(\cdot)$ under good events, which is also true for $w_1^u$). Following the procedure to derive equation \eqref{eq:D_m_ge_1_b_1} and \eqref{eq:D_m_ge_1_b_2}, $E_2$ can be bounded by
$$
    O\left(\exp(-\tau_j^2/4) (1 \vee \rho_{j1}) R_{j1}^* + \tau_j^2 / R_{j1}^* + D_m/R_{j1}^* \right) D_m,
$$
which the same choice of parameters $\tau_j = \Theta(\sqrt{\log (R_{j1}^* k \rho_\pi)}$ gives $E_2 \le c_b D_m / (k\rho_\pi)$ with the same SNR condition.

\paragraph{$j = 1$:} 
\paragraph{Bounding E1.} We define events
\begin{align*}
    \Eps_1 &= \{|\vdot{X}{\Delta_j}| \le D_m \tau, \forall j\} \cap \{|e| \le \tau \} \\
    \Eps_2 &= \{ |\vdot{X}{\beta_1^* - \beta_j^u}| \ge 4\tau, \forall j \neq 1 \}.
\end{align*}
For bounding $E_1$, same as when $D_m \ge 1$, $b_1 = 0$. Thus, we consider $b_2$ only, which is
\begin{align*}
    b_2 &= |\E_{\D_1} [\Delta_w \vdot{X}{s} e]| \\
    &\le \Bigg| \underbrace{\E_{\D_1} \Bigg[w_1^u(1-w_1^u) (\vdot{X}{\beta_1^* - \beta_1^u} + e) \vdot{X}{\Delta_1} \vdot{X}{s} e \Bigg]}_{d_1} \\
    &+ \underbrace{\E_{\D_1} \Bigg[ \sum_{l\neq1} w_1^u w_l^u (\vdot{X}{\beta_1^* - \beta_l^u} + e) \vdot{X}{\Delta_l} \vdot{X}{s} e \Bigg]}_{d_2} \Bigg|.
\end{align*}

First part of the proof follows the path for $j \neq 1$.
\begin{align*}
    d_2 &\le \E_{\D_1} \left[ \left|\sum_{l\neq 1} w_1^u w_l^u (\vdot{X}{\beta_1^* - \beta_l^u} + e) \vdot{X}{\Delta_l} \vdot{X}{s} e \indic_{\Eps_1} \right| \right] \\
    &\le \E_{\D_1} \left[\sum_{l \neq 1} \left| w_1^u \rho_\pi \exp(-6\tau^2) 4\tau \vdot{X}{\Delta_l} \vdot{X}{s} e \right| \indic_{\Eps_1} \right] + \E_{\D_1} \left[\sum_{l \neq 1} \left|w_1^u w_l^u 4\tau \vdot{X}{\Delta_l} \vdot{X}{s} e \right| \indic_{\Eps_1} \right]  \\
    &+ \E_{\D_1} \left[ \left| \sum_{l \neq 1} w_1^u w_l^u (\vdot{X}{\beta_1^* - \beta_l^u} + e) \vdot{X}{\Delta_l} \vdot{X}{s} e \right| \indic_{\Eps_1^c} \right] \\
    &\le 4 \rho_\pi D_m\tau^2 \exp(-6\tau^2) \underbrace{\E_{\D_1} \left[\sum_{l} \left| w_1^u \vdot{X}{s} e \right| \indic_{\Eps_1} \right] }_{(i)} + 4 D_m \tau^2 \underbrace{\E_{\D_1} \left[ \left| (\sum_{l\neq 1} w_l^u) w_1^u \vdot{X}{s} e \right| \indic_{\Eps_1} \right]}_{(ii)} \\
    &+ \underbrace{\E_{\D_1} \left[\left| \sum_{l \neq 1} w_1^u w_l^u (\vdot{X}{\beta_1^* - \beta_l^u} + e) \vdot{X}{\Delta_l} \vdot{X}{s} e\indic_{\Eps_1^c} \right| \right]}_{(iii)}.
\end{align*}

For (i), 
$$
    \E_{\D_1} \left[\sum_{l} \left| w_1^u \vdot{X}{s} e \right| \indic_{\Eps_1} \right] \le \sum_k \E_{\D_1} [|e \vdot{X}{s}|] \le k.
$$

(ii), we use event $\Eps_2$ as before,
\begin{align*}
    \E_{\D_1} \left[ \left| (1 - w_1^u) w_1^u \vdot{X}{s} e \right| \indic_{\Eps_1} \right] &= \E_{\D_1} \left[ \left| (1 - w_1^u) w_1^u \vdot{X}{s} e \right| \indic_{\Eps_1 \cap \Eps_2} \right] \\
    &+ \E_{\D_1} \left[ \left| (1 - w_1^u) w_1^u \vdot{X}{s} e \right| \indic_{\Eps_1 \cap \Eps_2^c} \right].
\end{align*}
Under event $\Eps_1 \cap \Eps_2$, it is now easy to show that $w_l^u \le 3 \rho_\pi \exp(-2 \tau^2)$ for all $l \neq 1$. Thus, $1 - w_1^u \le 3 k \rho_\pi \exp(-2 \tau^2)$, and 
$$
    \E_{\D_1} \left[ \left| (1 - w_1^u) w_1^u \vdot{X}{s} e \right| \indic_{\Eps_1 \cap \Eps_2} \right] \le 3 k \rho_\pi \exp(-2 \tau^2) \E_{\D_1} [|\vdot{X}{s} e|] \le 3 k \rho_\pi \exp(-2 \tau^2).
$$
For $\Eps_1 \cap \Eps_2^c$, 
\begin{align*}
    \E_{\D_1} \left[ \left| (1 - w_1^u) w_1^u \vdot{X}{s} e \right| \indic_{\Eps_1 \cap \Eps_2^c} \right] &\le \E_{\D_1} [|\vdot{X}{s} e| \indic_{\Eps_2^c}] \\
    &\le \sqrt{\E_{\D_1} [\vdot{X}{s}^2 | \Eps_2^c]} \sqrt{\E_{\D_1}[e^2| \Eps_2^c]} P(\Eps_2^c) \\
    &\le c_1 \sqrt{\log k} \frac{k \tau}{R_{min}}.
\end{align*}

For (iii), 
\begin{align*}
    (iii) &= \E_{\D_1} \left[\left| \sum_{l \neq 1} w_1^u w_l^u (\vdot{X}{\beta_1^* - \beta_l^u} + e) \vdot{X}{\Delta_l} \vdot{X}{s} e\indic_{\Eps_1^c} \right| \right] \\
    &\le \sum_{l \neq 1} \E_{\D_1} \left[\left| w_l^u (\vdot{X}{\beta_1^* - \beta_l^u} + e) \vdot{X}{\Delta_l} \vdot{X}{s} e\indic_{\Eps_1^c} \right| \right],
\end{align*}
\begin{align*}
    \E_{\D_1} \left[\left| w_l^u (\vdot{X}{\beta_1^* - \beta_l^u} + e) \vdot{X}{\Delta_l} \vdot{X}{s} e\indic_{\Eps_1^c} \right| \right] \le &\sqrt{\E_{\D_1}[(w_l^u)^2 (\vdot{X}{\beta_1^* - \beta_l^u} + e)^2 \vdot{X}{\Delta_l}^2]} \\
    & \sqrt[4]{\E_{\D_1}[\vdot{X}{s}^4 e^4]} \sqrt[4] {P(\Eps_1^c)}, 
\end{align*}

For bounding $\sqrt{\E_{\D_1}[(w_l^u)^2 (\vdot{X}{\beta_1^* - \beta_l^u} + e)^2 \vdot{X}{\Delta_l}^2]}$ for $l \neq 1$, we can again use Lemma \ref{lem:supplementary_Dm_less1}. We also have that $P(\Eps_1^c) \le k \exp(-\tau^2/2)$. Then, 
\begin{align*}
    (iii) \le c_2 k \sqrt[4]{k} \exp(-\tau^2/8) D_m.
\end{align*}
Combining all, 
\begin{equation}
    \label{eq:d_2_Dm_less_1_j_eq_1}
    d_2 \le O\left((k^{5/4} + k\tau^2) \exp(-\tau^2/8) + k \sqrt{\log k} \tau^3 / R_{min} \right) D_m.
\end{equation} 
Along with our choice $\tau = \Theta(\sqrt{\log(k\rho_\pi)})$ and $R_{min} = \tilde{\Omega}(k)$, we get $d_2 \le c_d D_m$.

For bounding $d_1$, (all constants $c_1, c_2, ...$ are renewed)
\begin{align*}
    d_1 &= \E_{\D_1} \left[w_1^u (1 - w_1^u) (\vdot{X}{\beta_1^* - \beta_1^u} + e) \vdot{X}{\Delta_1} \vdot{X}{s} e \right] \\
    &\le \E_{\D_1} \left[|w_1^u (1 - w_1^u) (\vdot{X}{\beta_1^* - \beta_1^u} + e) \vdot{X}{\Delta_1} \vdot{X}{s} e| \indic_{\Eps_1 \cap \Eps_2} \right] \\
    &+ \E_{\D_1} \left[|w_1^u (1 - w_1^u) (\vdot{X}{\beta_1^* - \beta_1^u} + e) \vdot{X}{\Delta_1} \vdot{X}{s} e| \indic_{\Eps_1^c} \right] \\
    &+ \E_{\D_1} \left[|w_1^u (1 - w_1^u) (\vdot{X}{\beta_1^* - \beta_1^u} + e) \vdot{X}{\Delta_1} \vdot{X}{s} e| \indic_{\Eps_1 \cap \Eps_2^c} \right] \\
    &\le k \rho_\pi \exp(-\tau^2/2) \underbrace{\E_{\D_1} \left[|(|u\vdot{X}{\Delta_1}| + |e|) \vdot{X}{\Delta_1} \vdot{X}{s} e| \indic_{\Eps_1 \cap \Eps_2} \right]}_{(i)} \\
    &+ \underbrace{\E_{\D_1} \left[(|u\vdot{X}{\Delta_1}| + |e|) \vdot{X}{\Delta_1} \vdot{X}{s} e| \indic_{\Eps_1^c} \right]}_{(ii)} \\
    &+ \underbrace{\E_{\D_1} \left[(|u\vdot{X}{\Delta_1}| + |e|) \vdot{X}{\Delta_1} \vdot{X}{s} e| \indic_{\Eps_1 \cap \Eps_2^c} \right]}_{(iii)}.
\end{align*}

\begin{align*}
    (i) &= \E_{\D_1} \left[|(|u\vdot{X}{\Delta_1}| + |e|) \vdot{X}{\Delta_1} \vdot{X}{s} e| \indic_{\Eps_1 \cap \Eps_2} \right] \\
    &\le \E_{\D_1} \left[\vdot{X}{\Delta_1}^2 |\vdot{X}{s} e| \right] + \E_{\D_1} |\left[\vdot{X}{\Delta_1} \vdot{X}{s} e^2| \right] \\
    &\le c_1\| \Delta_1 \| (1 + \|\Delta_1 \|) \le 2c_1 D_m.
\end{align*}

\begin{align*}
    (ii) &\le \E_{\D_1} \left[\vdot{X}{\Delta_1}^2 |\vdot{X}{s} e| \indic_{\Eps_1^c} \right] + \E_{\D_1} \left[|\vdot{X}{\Delta_1} \vdot{X}{s} e^2| \indic_{\Eps_1^c} \right] \\
    &= \sqrt{\E_{\D_1} [\vdot{X}{\Delta_1}^4 \vdot{X}{s}^2 e^2]} \sqrt{P(\Eps_1^c)} + \sqrt{\E_{\D_1} [\vdot{X}{\Delta_1}^2 \vdot{X}{s}^2 e^4]} \sqrt{P(\Eps_1^c)}\\
    &\le c_2 \sqrt{k} \| \Delta_1 \| \exp(-\tau^2/4).
\end{align*}

\begin{align*}
    (iii) &\le D_m^2 \tau^2 \sqrt{\E_{\D_1} [\vdot{X}{s}^2 | \Eps_2^c]} \sqrt{\E_{\D_1} [e^2 | \Eps_2^c]} P(\Eps_2^c) \\
    &+ D_m \tau \sqrt{\E_{\D_1} [\vdot{X}{s}^2 | \Eps_2^c]} \sqrt{\E_{\D_1} [e^4 | \Eps_2^c]} P(\Eps_2^c) \\
    &\le c_3 \sqrt{ \log k } D_m \frac{k\tau^3}{R_{min}},
\end{align*}
where we applied Corollary \ref{corollary:key_bound}. (i), (ii), (iii) gives a bound for $d_1$ as 
\begin{equation}
    \label{eq:d_1_Dm_less_1_j_eq_1}
    d_1 \le O \left(k \rho_\pi \exp(-\tau^2/4) + k \sqrt{\log k} \tau^3 / R_{min} \right) D_m.
\end{equation}
Now combining \eqref{eq:d_2_Dm_less_1_j_eq_1} and \eqref{eq:d_1_Dm_less_1_j_eq_1} we get the bound for $E_1 \le c_e D_m$, with the choice of $\tau = \Theta(\sqrt{\log (k\rho_\pi)})$ and high SNR $\tilde{\Omega(k)}$. 

\paragraph{Bounding $E_2$, the term from mismatch in mixing weights.} When $j = 1$,
\begin{align*}
    \Delta_{w,2} &= -w_1^u(1-w_1^u) \delta_1 / \pi_1^u + \sum_{l\neq1} w_1^u w_l^u \delta_l / \pi_l^u \le \left| w_1^u(1-w_1^u) + \sum_{l\neq1} w_1^u w_l^u \right| D_m = 2 w_1^u (1 - w_1^u) D_m.
\end{align*}
Hence, $E_2 \le D_m \E_{\D_j} [(1-w_1^u) |\vdot{X}{s} (Y - \vdot{X}{\beta_1^*})|]$. Again, we have already seen similar equation when we handle $D_m \ge 1$. Following the procedure to derive equation \eqref{eq:b2_Dm_greater_1_j_eq_1}, $E_2$ can be bounded by
$$
    O\left(k \rho_\pi \exp(-\tau^2/4) + (k\sqrt{\log k}) \tau / R_{min} + (k\sqrt{\log k}) D_m/R_{min} \right) D_m,
$$
which the same choice of parameters $\tau = \Theta(\sqrt{\log (k \rho_\pi)})$ gives $E_2 \le c_b D_m$ with the SNR condition $\tilde{\Omega}(k)$.

Summing up everything, for $j \neq 1$ we have $B_j \le O(D_m / (k\rho_\pi))$, and for $j = 1$ we have $B_1 \le O(D_m)$. Thus, $B \le \pi_1^* B_1 + \sum_{j \neq 1} \pi_j^* B_j \le c_B D_m \pi_1^*$ for some constant $c_B \in (0, 1/8)$ by properly setting constants in the proof. That is $\|\beta_1^+ - \beta_1^*\| \le c_B D_m \pi_1^*$.

\paragraph{Update for mixing weights.} The procedure is exactly same for proving the bound for $\|\beta_1^+ - \beta_1^*\|$. It is actually easier since it does not involve additional terms $\vdot{X}{s}$ and $Y - \vdot{X}{\beta_1^*}$ as can be seen in \eqref{eq:weight_update}. Thus we can follow the exact same procedure, getting $|\pi_1^+ - \pi_1^*| / \pi_1^* \le c_B D_m$.

\paragraph{Proof of Lemma \ref{lem:supplementary_Dm_less1}}
\label{Appendix:population_supp}
\begin{proof}
    If $j \neq l$, we define a new event with new parameter $\tau_l$,
    \begin{align*}
        \Eps_{1,l} = \{ |\vdot{X}{\Delta_l}| \le D_m \tau_l \} \cap \{ |e| \le \tau_l \} \\
        \Eps_{2,l} = \{ |\vdot{X}{\beta_j^* - \beta_l^u}| \ge 4 \tau_l \}.
    \end{align*}
    Under event $\Eps_{1,l}$, we can show that 
    $$|w_l^u|^2 \vdot{X}{(\beta_j^* - \beta_l^u + e)}^2 \indic_{\Eps_{1,l}} \le (\rho_{jl} \exp(-6\tau_l^2) 4\tau_l)^2 \indic_{\Eps_{1,l} \cap \Eps_{2,l}} + (w_l^u 4\tau_l)^2 \indic_{\Eps_{1,l} \cap \Eps_{2,l}^c}.$$
    Now we can bound \eqref{eq:inner_Eps_l} as,
    \begin{align*}
        \E_{\D_j} &\left[(w_l^u)^2 \vdot{X}{(\beta_j^* - \beta_l^u + e)}^2\vdot{X}{\Delta_l}^2\right] \\
        &\le \E_{\D_j} \left[16 \rho_{jl} \exp(-12\tau_l^2) \tau_l^2 \vdot{X}{\Delta_l}^2 \indic_{\Eps_{1,l} \cap \Eps_{2,l}} \right] \\
        &+ \E_{\D_j} \left[16 (w_l^u)^2 \tau_l^2 \vdot{X}{\Delta_l}^2 \indic_{\Eps_{1,l} \cap \Eps_{2,l}^c} \right] \\
        &+ \E_{\D_j} \left[(w_l^u)^2 \vdot{X}{(\beta_j^* - \beta_l^u + e)}^2\vdot{X}{\Delta_l}^2 \indic_{\Eps_{1,l}^c} \right].
    \end{align*}
    
    We do similarly bound each term:
    \begin{align*}
        \E_{\D_j} \left[16 \exp(-12\tau_l^2) \tau_l^2 \vdot{X}{\Delta_l}^2 \indic_{\Eps_{1,l} \cap \Eps_{2,l}} \right] &\le c_1 \rho_{jl} \exp(-12\tau_l^2) \tau_l^2 \|\Delta_l\|^2,
    \end{align*}
    \begin{align*}
        \E_{\D_j} \left[16 (w_l^u)^2 \tau_l^2 \vdot{X}{\Delta_l}^2 \indic_{\Eps_{1,l} \cap \Eps_{2,l}^c} \right] &\le 16 \tau_l^2 \E_{\D_j} \left[(w_l^u)^2 \vdot{X}{\Delta_l}^2 \indic_{\Eps_{2,l}^c} \right] \\
        &\le 16 \tau_l^2 \E_{\D_j} \left[\vdot{X}{\Delta_l}^2|\Eps_{2,l}^c \right]  P(\Eps_{2,l}^c) \\
        &\le c_2 \tau_l^2 \| \Delta_l \|^2 \tau_l/R_{jl}^*,
    \end{align*}
    \begin{align*}
        \E_{\D_j} &\left[[(w_l^u)^2 \vdot{X}{(\beta_j^* - \beta_l^u + e)}^2 \vdot{X}{\Delta_l}^2 \indic_{\Eps_{1,l}^c} \right] \\
        &\le \E_{\D_j} \left[2\vdot{X}{\beta_j^* - \beta_l^u}^2\vdot{X}{\Delta_l}^2\indic_{\Eps_{1,l}^c} \right] 
        + \E_{\D_j} \left[ 2e^2\vdot{X}{\Delta_l}^2 \indic_{\Eps_{1,l}^c}\right] \\
        &\le 2 \sqrt{\E_{\D_j} \left[\vdot{X}{\beta_j^* - \beta_l^u}^4 \vdot{X}{\Delta_l}^4 \right]} \sqrt{P(\Eps_{1,l}^c)} + 2 \sqrt{\E_{\D_j} \left[e^4 \vdot{X}{\Delta_l}^4 \right]} \sqrt{P(\Eps_{1,l}^c)}  \\
        &\le c_3 (R_{jl}^*)^2 \|\Delta_l\|^2 \exp(-\tau_l^2/2) + c_4 \|\Delta_l\|^2 \exp(-\tau_l^2/2).
    \end{align*}
    Set $\tau_l = \Theta(\sqrt{\log (R_{jl}^* \rho_\pi)})$. Then every terms will be canceled out and we get
    $$
        \eqref{eq:inner_Eps_l} \le O( \| \Delta_l \|^2 ).
    $$
    
    If $l = j$, then 
    \begin{align*}
        \E_{\D_j} &\left[(w_l^u)^2 \vdot{X}{(\beta_j^* - \beta_l^u + e)}^2\vdot{X}{\Delta_l}^2\right] \\
        &\le \E_{\D_j} \left[4 \tau_l^2 \vdot{X}{\Delta_l}^2 \indic_{\Eps_{1,l}} \right] \\
        &+ \E_{\D_j} \left[(\vdot{X}{\Delta_l} + e)^2 \vdot{X}{\Delta_l}^2 \indic_{\Eps_{1,l}^c} \right].
    \end{align*}
    Each term is easy to bound.
    \begin{align*}
        \E_{\D_j} \left[4 \tau_l^2 \vdot{X}{\Delta_l}^2 \indic_{\Eps_{1,l}} \right] &\le O(\tau_l^2 D_m^2). \\
        \E_{\D_j} \left[(\vdot{X}{\Delta_l} + e)^2 \vdot{X}{\Delta_l}^2 \indic_{\Eps_{1,l}^c} \right] &\le \E_{\D_j} \left[2\vdot{X}{\Delta_l}^4 + 2e^2 \vdot{X}{\Delta_l}^2 \indic_{\Eps_{1,l}^c} \right] \\
        &\le 2 \sqrt{\E_{\D_j} \left[\vdot{X}{\Delta_l}^8 \right]} \sqrt{P(\Eps_{1,l}^c)} + 2 \sqrt{\E_{\D_j} \left[ e^4 \vdot{X}{\Delta_l}^4 \right]} \sqrt{P(\Eps_{1,l}^c)} \\
        &\le O((\| \Delta_l \|^4 + \| \Delta_l \|^2) \sqrt{k} \exp(-\tau_l^2/4)). 
    \end{align*}
    We set $\tau_l = O(\sqrt{\log k})$ and get 
    $$
        \eqref{eq:inner_Eps_l} \le O(\| \Delta_l \|^2 \log k).
    $$
\end{proof}

%% file: main.bbl
\begin{thebibliography}{10}

\bibitem{wu1983convergence}
CF~Jeff Wu et~al.
\newblock On the convergence properties of the em algorithm.
\newblock {\em The Annals of statistics}, 11(1):95--103, 1983.

\bibitem{pmlr-v99-kwon19a}
Jeongyeol Kwon, Wei Qian, Constantine Caramanis, Yudong Chen, and Damek Davis.
\newblock Global convergence of the em algorithm for mixtures of two component
  linear regression.
\newblock In {\em 32nd Annual Conference on Learning Theory}, pages 2055--2110.
  PMLR, 2019.

\bibitem{balakrishnan2017statistical}
Sivaraman Balakrishnan, Martin~J Wainwright, Bin Yu, et~al.
\newblock Statistical guarantees for the em algorithm: From population to
  sample-based analysis.
\newblock {\em The Annals of Statistics}, 45(1):77--120, 2017.

\bibitem{klusowski2019estimating}
Jason~M Klusowski, Dana Yang, and WD~Brinda.
\newblock Estimating the coefficients of a mixture of two linear regressions by
  expectation maximization.
\newblock {\em IEEE Transactions on Information Theory}, 2019.

\bibitem{yi2014alternating}
Xinyang Yi, Constantine Caramanis, and Sujay Sanghavi.
\newblock Alternating minimization for mixed linear regression.
\newblock In {\em International Conference on Machine Learning}, pages
  613--621, 2014.

\bibitem{jin2016local}
Chi Jin, Yuchen Zhang, Sivaraman Balakrishnan, Martin~J Wainwright, and
  Michael~I Jordan.
\newblock Local maxima in the likelihood of gaussian mixture models: Structural
  results and algorithmic consequences.
\newblock In {\em Advances in neural information processing systems}, pages
  4116--4124, 2016.

\bibitem{daskalakis2017ten}
Constantinos Daskalakis, Christos Tzamos, and Manolis Zampetakis.
\newblock Ten steps of em suffice for mixtures of two gaussians.
\newblock In {\em 30th Annual Conference on Learning Theory}, 2017.

\bibitem{zhao2018statistical}
Ruofei Zhao, Yuanzhi Li, and Yuekai Sun.
\newblock Statistical convergence of the em algorithm on gaussian mixture
  models.
\newblock {\em arXiv preprint arXiv:1810.04090}, 2018.

\bibitem{yan2017convergence}
Bowei Yan, Mingzhang Yin, and Purnamrita Sarkar.
\newblock Convergence of gradient em on multi-component mixture of gaussians.
\newblock In {\em Advances in Neural Information Processing Systems}, pages
  6956--6966, 2017.

\bibitem{yi2016solving}
Xinyang Yi, Constantine Caramanis, and Sujay Sanghavi.
\newblock Solving a mixture of many random linear equations by tensor
  decomposition and alternating minimization.
\newblock {\em arXiv preprint arXiv:1608.05749}, 2016.

\bibitem{zhong2016mixed}
Kai Zhong, Prateek Jain, and Inderjit~S Dhillon.
\newblock Mixed linear regression with multiple components.
\newblock In {\em Advances in neural information processing systems}, pages
  2190--2198, 2016.

\bibitem{hand2018convex}
Paul Hand and Babhru Joshi.
\newblock A convex program for mixed linear regression with a recovery
  guarantee for well-separated data.
\newblock {\em Information and Inference: A Journal of the IMA}, 7(3):563--579,
  2018.

\bibitem{sedghi2014provable}
Hanie Sedghi, Majid Janzamin, and Anima Anandkumar.
\newblock Provable tensor methods for learning mixtures of generalized linear
  models.
\newblock {\em Proceedings of Machine Learning Research}, 51:1223--1231, 2014.

\bibitem{chaganty2013spectral}
Arun~Tejasvi Chaganty and Percy Liang.
\newblock Spectral experts for estimating mixtures of linear regressions.
\newblock In {\em International Conference on Machine Learning}, pages
  1040--1048, 2013.

\bibitem{chen2014convex}
Yudong Chen, Xinyang Yi, and Constantine Caramanis.
\newblock A convex formulation for mixed regression with two components:
  Minimax optimal rates.
\newblock In {\em Conference on Learning Theory}, pages 560--604, 2014.

\bibitem{chen2017convex}
Yudong Chen, Xinyang Yi, and Constantine Caramanis.
\newblock Convex and nonconvex formulations for mixed regression with two
  components: Minimax optimal rates.
\newblock {\em IEEE Transactions on Information Theory}, 64(3):1738--1766,
  2017.

\bibitem{li2018learning}
Yuanzhi Li and Yingyu Liang.
\newblock Learning mixtures of linear regressions with nearly optimal
  complexity.
\newblock In {\em Conference On Learning Theory}, pages 1125--1144, 2018.

\bibitem{yi2015regularized}
Xinyang Yi and Constantine Caramanis.
\newblock Regularized em algorithms: A unified framework and statistical
  guarantees.
\newblock In {\em Advances in Neural Information Processing Systems}, pages
  1567--1575, 2015.

\bibitem{regev2017learning}
Oded Regev and Aravindan Vijayaraghavan.
\newblock On learning mixtures of well-separated gaussians.
\newblock In {\em 2017 IEEE 58th Annual Symposium on Foundations of Computer
  Science (FOCS)}, pages 85--96. IEEE, 2017.

\bibitem{cai2019chime}
T~Tony Cai, Jing Ma, Linjun Zhang, et~al.
\newblock Chime: Clustering of high-dimensional gaussian mixtures with em
  algorithm and its optimality.
\newblock {\em The Annals of Statistics}, 47(3):1234--1267, 2019.

\bibitem{vershynin2010introduction}
Roman Vershynin.
\newblock Introduction to the non-asymptotic analysis of random matrices.
\newblock {\em arXiv preprint arXiv:1011.3027}, 2010.

\end{thebibliography}
